\documentclass[compsoc,conference,a4paper,10pt]{IEEEtran}
\IEEEoverridecommandlockouts
\usepackage[utf8]{inputenc} 
\usepackage{amsmath}
\usepackage{hyperref}    
\usepackage{url}            
\usepackage{nicefrac}
\usepackage{booktabs}     
\usepackage{nicefrac}       
\usepackage{xcolor}   
\usepackage[T1]{fontenc}
\usepackage{graphicx}

\usepackage{subcaption}

\usepackage{algorithm}
\usepackage{algpseudocode}
\usepackage{stmaryrd}
\usepackage{enumitem}
\usepackage{url}
\usepackage{multirow}
\usepackage{amsmath,amssymb,amsthm}
\usepackage{mathrsfs}
\usepackage{natbib}
\usepackage{caption}
\usepackage{tabularx} 
\usepackage{color}
\usepackage[para]{threeparttable}
\usepackage{soul}
\usepackage{pifont} 
\usepackage{xspace}
\usepackage{epstopdf}
\usepackage{amssymb}
\usepackage{tikz}
\usepackage{arydshln}
\usepackage[dvipsnames]{xcolor}
\usepackage[svgnames]{xcolor}

\newcolumntype{L}{>{\raggedright\arraybackslash}X}

\newtheorem{theorem}{Theorem}
\newtheorem{lemma}[theorem]{Lemma}
\newtheorem{corollary}[theorem]{Corollary}
\newtheorem{definition}{Definition}

\usepackage{array}    

\def\eg{\emph{e.g.}}
\def\ie{\emph{i.e.}}
\def\cf{\emph{c.f.}}
\newcommand{\etal}{\textit{et al.}\xspace}

\newcommand{\folktables}{$\mathsf{FolkTables}$}
\newcommand{\femnist}{$\mathsf{FEMNIST}$}
\newcommand{\cifar}{$\mathsf{CIFAR10}$}
\newcommand{\GTSRB}{$\mathsf{GTSRB}$}

\newcommand{\scheme}{\mathcal{E}}
\newcommand{\keygen}{\textsf{KeyGen}}
\newcommand{\enc}{\textsf{Enc}}
\newcommand{\dec}{\textsf{Dec}}
\newcommand{\eval}{\textsf{Eval}}

\newcommand{\pk}{\textsf{pk}}
\newcommand{\sk}{\textsf{sk}}
\newcommand{\evk}{\textsf{evk}}

\newcommand{\redcross}{\textcolor{red}{\ding{53}}}

\newcommand{\thetaJt}{\theta_j^t}
\newcommand{\thetaTplusOne}{\theta^{t+1}}

\newcommand{\DeltaJt}{\Delta_j^t}
\newcommand{\PJt}{P_j^t}
\newcommand{\thetaPrimeJ}{\theta'_j}

\newcommand{\thetaTwoPrimeTplusOne}{\theta^{t+1}}
\newcommand{\STplusOne}{S^{t+1}}
\newcommand{\SinvTplusOne}{S_{\text{inv}}^{t+1}}
\newcommand{\thetaNewJ}{\theta_j^{t+1}}
\newcommand{\FilterP}{\mathsf{Filter}_P}
\newcommand{\EncInv}{\mathsf{FHEInverse}} 
\newcommand{\Boot}{\mathsf{Boot}}     
\newcommand{\LocalUpdate}{\mathsf{LocalUpdate}}
\newcommand{\ThKeyGen}{\mathsf{Th.KeyGen}}
\newcommand{\cD}{\mathcal{D}}

\DeclareMathOperator{\E}{\mathbb{E}}
\newcommand{\norm}[1]{\left\lVert#1\right\rVert}

\newcommand{\fedavg}{$\mathsf{FedAvg}$}

\newcommand{\acsincome}{$\mathsf{ACSIncome}$}

\begin{document}

\title{Robust Federated Learning via Byzantine Filtering over Encrypted Updates}

\author{
    \IEEEauthorblockN{
        Adda Akram Bendoukha\IEEEauthorrefmark{1},
        Aymen Boudguiga\IEEEauthorrefmark{2},
        Nesrine Kaaniche\IEEEauthorrefmark{1}, \\
        Renaud Sirdey\IEEEauthorrefmark{2},
        Didem Demirag\IEEEauthorrefmark{3} and
        Sébastien Gambs\IEEEauthorrefmark{3}
    }
    \IEEEauthorblockA{
        \IEEEauthorrefmark{1}Samovar, Télécom SudParis, Institut Polytechnique de Paris, France\\
        Email: \{adda-akram.bendoukha, kaaniche.nesrine\}@telecom-sudparis.eu
    }
    \IEEEauthorblockA{
        \IEEEauthorrefmark{2}CEA-List, Université Paris-Saclay, France\\
        Email: \{aymen.boudguiga, renaud.sirdey\}@cea.fr
    }
    \IEEEauthorblockA{
        \IEEEauthorrefmark{3}Université du Québec à Montréal, Canada\\
        Email: \{demirag.didem, gambs.sebastien\}@uqam.ca
    }
}

\maketitle

\begin{abstract}
Federated Learning (FL) aims to train a collaborative model while preserving data privacy. 
However, the distributed nature of this approach still raises privacy and security issues, such as the exposure of sensitive data due to inference attacks and the influence of Byzantine behaviors on the trained model. 
In particular, achieving both secure aggregation and Byzantine resilience remains challenging, as existing solutions often address these aspects independently. 
In this work, we propose to address these challenges through a novel approach that combines homomorphic encryption for privacy-preserving aggregation with property-inference-inspired meta-classifiers for Byzantine filtering.
First, following the property-inference attacks blueprint, we train a set of filtering meta-classifiers on labeled shadow updates, reproducing a diverse ensemble of Byzantine misbehaviors in FL, including backdoor, gradient-inversion, label-flipping and shuffling attacks. 
The outputs of these meta-classifiers are then used to cancel the Byzantine encrypted updates by reweighting.
Second, we propose an automated method for selecting the optimal kernel and the dimensionality hyperparameters with respect to homomorphic inference, aggregation constraints and efficiency over the CKKS cryptosystem.
Finally, we demonstrate through extensive experiments the effectiveness of our approach against Byzantine participants on the \femnist{}, \cifar{}, \GTSRB{}, and \acsincome{} benchmarks. 
More precisely, our SVM filtering achieves accuracies between $90$\% and $94$\% for identifying Byzantine updates at the cost of marginal losses in model utility and encrypted inference runtimes ranging from $6$ to $24$ seconds and from $9$ to $26$ seconds for an overall aggregation.
\end{abstract}


\section{Introduction}  
\label{sec:introduction} 
 
Federated Learning (FL)~\cite{FedAvg2017} is a well-established framework for collaboratively training machine learning models. 
In FL, individual workers, termed clients, train a shared model locally on their private datasets and then delegate the model aggregation to a central server. 
In a nutshell, starting from an initial common model, the workers use their local training data to update the model locally at each training round. 
Then, these updates are transmitted and aggregated by the server, resulting in a new model that forms the input for the next round. 
At the end of the training, the resulting model is expected to display a predictive performance comparable to that of an ideal model trained on the union of the workers' datasets. 
The distributed nature of FL, while providing some form of privacy protection and communication efficiency, induces significant trust challenges. 
Indeed, transmitted updates can be exploited to perform gradient inversion attacks~\cite{Nasr_2019, lyu2022privacy, Rodr_guez_Barroso_2023} or reconstruct sensitive information about local datasets~\cite{parisot2021propertyinferenceattacksconvolutional, kerkouche2023clientspecificpropertyinferencesecure, bendoukha2025}. 
Two main encryption techniques have been proposed to protect workers' local updates, namely Secure Multiparty Computation (MPC)~\cite{li2020privacy, byrd2020differentially} and Fully Homomorphic Encryption (FHE)~\cite{9682053, sébert2022protecting}. 
MPC ensures the collaborative computation of a shared function over local updates, whereas FHE enables the processing of encrypted updates without decryption.  
Recent studies have shown that FHE offers promising performance for privacy-preserving aggregation while supporting complex aggregation rules~\cite{stan2022, EURECOM+6974, bendoukha2025}.
Although encryption techniques protect workers' updates during the training phase, they do not mitigate post-deployment privacy attacks, such as membership inference, attribute reconstruction, and model inversion~\cite{shokri2017membership, rigaki2020survey,lyu2022privacy}. 
Nonetheless, Differential Privacy (DP) is a well-established countermeasure against these attacks and is usually co-implemented with encryption techniques in FL.

However, implementing secure aggregation in FL gives rise to a critical integrity-privacy trade-off. 
On the one hand, FL must protect the workers' updates from honest-but-curious entities (\emph{i.e.}, the server and other workers). 
On the other hand, FL must ensure the integrity of the models' updates, especially against Byzantine workers, which aim to disrupt the model convergence by providing faulty updates~\cite{blanchard2017machine} or introducing backdoors~\cite{li2023learningbackdoorfederatedlearning, nguyen2023backdoorattacksdefensesfederated}.
Secure aggregation techniques often overlook Byzantine resilience, generally considering it as a separate constraint~\cite{EURECOM+6974}. 
Conversely, many Byzantine-resilient aggregation schemes \cite{blanchard2017machine, lin2019freeridersfederatedlearningattacks, fraboni2021freeriderattacksmodelaggregation} do not integrate privacy-preserving mechanisms. 
These siloed approaches, which address either secure aggregation or Byzantine resilience, result in non-optimal trade-offs. 
In the best case, it leads to a significant computational overhead or reduced accuracy~\cite{choffrut2023sablesecurebyzantinerobust}, while in the worst case, it results in fundamental incompatibilities between privacy and resilience \cite{guerraoui2021differentialprivacybyzantineresilience, survey_byzantine_resilient_aggreg}. 

\textbf{Contributions.} To address these challenges, we propose a Byzantine-resilient FL approach that also ensures the privacy of the workers' updates through homomorphic encryption. 
To achieve this, we design a property inference meta-classifier capable of detecting Byzantine behavior directly from encrypted updates. 
More precisely, we consider a \fedavg{}-based FL with homomorphic encryption for securing workers' updates. 
Our main contributions can be summarized as follows. 
\begin{enumerate}[leftmargin=*]
    \item We propose a property-inference methodology to detect Byzantine workers from encrypted FL updates.  
    For this purpose, we design a set of SVM meta-classifiers to infer the statistical properties of different malicious behaviors, including backdoor, gradient-inversion and label-flipping attacks. 
    Our meta-classifier filtering approach can accurately identify several Byzantine behaviors, achieving up to 94\% on $\mathsf{F1}$-score. 
    \item The SVM meta-classifiers are optimized using a grid search across various model architectures and parameter embeddings, leading to efficient inference on homomorphically encrypted updates.
    To achieve this, we provide a complete set of homomorphic encryption building blocks necessary for evaluating our filtering approach on encrypted updates (\emph{cf.} Section~\ref{sec:fhe_efficient_filters}).
    \item We conduct extensive experiments to demonstrate the performance of our approach against Byzantines on the \femnist{}, \cifar{}, \GTSRB{} and \acsincome{} datasets, along with ResNet \cite{resnet}, VGG \cite{vgg}, custom CNN and MLP model architectures \footnote{Our implementations are available at \url{https://github.com/Akram275/FL_with_FHE_filtering}}. 
    For instance, our filtering approach achieves F1-scores between $90$\% and $94$\%, at the cost of marginal losses in model utility and encrypted inference runtime ranging from $3$ to $7$ seconds for a global aggregation (\emph{cf.} Section~\ref{sec:experimental}).
    \item 
    Additionally, we investigate the performance of our approach in a complex attack setting, in which two groups of adversaries employ distinct Byzantine strategies, \ie{}, data poisoning via backdoor injection and model poisoning via gradient ascent. 
    Our approach filters both attack updates with 96\% $\mathsf{F1}$-score while ensuring convergence (Appendix \ref{sec:appendix_single_filter_multiple_attacks}).
    Second, we demonstrate that our filtering approach remains effective even when updates are protected with DP, to thwart post-deployment risks like membership inference \cite{rigaki2020survey, shokri2017membership}.  
    This is not surprising as DP has been proven to provide limited protection against property inference attacks~\cite{naseri2020local}.
    We precisely leverage this to detect and handle malicious behaviors, thus increasing the FL robustness (Appendix \ref{sec:appendix_filter_perf_with_dp}).
\end{enumerate}

\textbf{Outline \& Notations.} 
First, Section~\ref{sec:related} provides an overview of prior work addressing malicious behavior in FL and Section~\ref{sec:background} introduces the key concepts underlying our proposed system.
Afterwards, Section~\ref{sec:relevant_properties_for_fl} analyzes various workers' misbehaviors in FL, and introduces our considered filtering properties, before presenting our encrypted property filtering framework in Section \ref{sec:our_approach}.
Then, Section~\ref{sec:filtering_data_gen} discusses the filtering data generation methodology, Section \ref{sec:fhe_efficient_filters} presents the FHE-friendly design of property filters and Section \ref{sec:ckks_building_blocks_for_filtering} describes the property-filtering framework based on CKKS. 
Finally, Section~\ref{sec:experimental} reports our experimental results, before concluding in Section~\ref{sec:conclusion}. Table \ref{tab:notations} provides the main notations.

\begin{table}[ht]
\centering
\caption{Notation summary}
\label{tab:notations}
\footnotesize
\renewcommand{\arraystretch}{1.2}
\begin{tabular}{ll}
\toprule
Symbol & Description \\
\midrule
$n = n_h + n_b$ & Total number of workers, with $n_h$ honest and $n_b$ Byzantine \\
$H, B$ & Sets of honest and Byzantine worker indices \\
$q = |B|/n$ & Fraction of Byzantine workers \\
$\theta^t$ & Global model parameters at round $t$ \\
$\theta^t_i$ & Local model of worker $i$ at round $t$ \\
$\Delta^t_i$ & Update difference: $\theta^{t+1}_i - \theta^t$ \\
$P^i_t \in [0,1]$ & Filter weight assigned to worker $i$ at round $t$ \\
$\nabla F(\theta)$ & Gradient of $F$ at $\theta$ \\
$T$ & Number of FL communication rounds \\
\bottomrule
\end{tabular}
\end{table}

\section{Related work}
\label{sec:related}
Several works have investigated the workers' misbehaviors in FL, which we categorize into two classes: outlier exclusion and proof-based verification. 
In this section, we first review these works and discuss their pros and cons when co-implemented with secure aggregation techniques such as FHE and DP. 
Then, we review Property Inference Attacks (PIA) commonly used in privacy auditing. 

\subsection{Outlier exclusion} 

Outlier exclusion methods identify Byzantine updates by detecting deviations based on predefined criteria, such as gradient norm or direction, while assuming that the majority of participating workers are honest. 
For example, \emph{Krum}~\cite{blanchard2017machine} computes the Euclidean distance between workers' gradients and discards the most distant ones before aggregation. 
Similarly, \emph{Trimmed-Mean} and \emph{Coordinate-Wise Trimmed-Mean (CWTM)}~\cite{yin2021byzantinerobustdistributedlearningoptimal} sort respectively gradients norm-wise and coordinate-wise, before eliminating the $f$ most deviant ones, in which $f$ is an estimate of the number of Byzantine workers. 
\emph{Clipped clustering}~\cite{karimireddy2021learninghistorybyzantinerobust} combines centered clipping and iterative clustering to group updates around a central point, to mitigate the influence of adversarial updates. \emph{FLTrust}~\cite{cao2022fltrustbyzantinerobustfederatedlearning} enhances Byzantine resilience by assigning trust scores to workers based on the directional similarity between their submitted gradients and a reference gradient computed by the server using a trusted root dataset. 
These scores are then used to re-weight workers' updates, thus reducing the impact of Byzantine contributions.

\paragraph{\textbf{Limitations}}
The accuracy of outlier exclusion methods decreases with non-Independently and Identically Distributed (non- IID) workers' data \cite{survey_byzantine_resilient_aggreg, kairouz2021advances}.
In such cases, outlier updates may stem from skewed local data, instead of malicious behaviors, and discarding them impacts the global model's generalization. 
The reliance on an honest majority, common in many outlier exclusion techniques, also makes them particularly vulnerable against Sybil attacks in FL~\cite{fung2020mitigatingsybilsfederatedlearning}. 
While \emph{FLTrust} avoids the honest-majority assumption, it involves the availability of a statistically representative dataset on the server side. 
Possessing such central data potentially negates the main purpose of using FL in the first place.  
Finally, these methods are primarily designed for gradient aggregation (\emph{e.g.}, $\mathsf{FedSGD}$), and have limited applicability to model aggregation (\textsf{FedAvg}) as they rely on vector-based geometric principles. 
\fedavg{}-based FL is more efficient because it operates epoch-wise.
Therefore, it requires fewer communication rounds (hence less secure aggregations) compared to $\mathsf{FedSGD}$-based FL, which operates step-wise. 

\paragraph{\textbf{Privacy-preserving outliers exclusion}}
Several works have explored privacy-preserving outlier exclusion methods~\cite{choffrut2023sablesecurebyzantinerobust, 9849010, MA2022103561}. 
Choffrut \etal~\cite{choffrut2023sablesecurebyzantinerobust} evaluated median-based aggregation circuits with BFV \cite{BFV2012}, BGV \cite{BGV2011} and CKKS \cite{CKKS_2017} encrypted updates. 
Xu \etal~\cite{MA2022103561} implemented Multi-Krum over the additive Paillier cryptosystem (AHE)~\cite{Pailler_cryptosystem}. 
However, due to the limited homomorphic properties of the Paillier cryptosystem, their approach requires interactivity with workers. 
Wang \etal~\cite{9685821} achieved Byzantine resilience via a blockchain-based distributed learning with CKKS for secure-aggregation and a cosine-based similarity for Byzantine resilience. 
These works highlight the substantial computational overhead associated with Byzantine-resilient aggregation in the homomorphic domain. 
This overhead stems primarily from the need to implement vector sorting circuits in FHE—necessary for median-based robust aggregation rules, which ultimately results in poor scalability to the number of workers.
In \cite{guerraoui2021differentialprivacybyzantineresilience}, Guerraoui \etal investigated the application of DP to these aggregation strategies using the variance-to-norm (VN) ratio~\cite{blanchard2017machine}. 
Their analysis demonstrates that enforcing strong DP ($\epsilon \leq 0.5$) at the update level requires stringent conditions, such as excessively large batch sizes to satisfy the VN ratio required for convergence.

\subsection{Proof-based verification} 

Proof-based verification systems aim to validate the integrity of ML tasks. 
Namely, training \cite{jia2021proof, shamsabadi2024confidential, abbaszadeh2024zero} and inference \cite{yadav2024fairproof}. 
In the context of training,  Jia \etal~\cite{jia2021proof} introduced the \emph{Proof-of-Learning} (PoL) framework, attesting that the model's parameters result from a legitimate optimization process.
However, their initial approach compromises the model's privacy by revealing intermediate weights, which are used during verification. 
In addition, the proof size significantly exceeds that of the model's parameters, potentially involving billions of values. 
Subsequent research explored the application of \emph{Zero-Knowledge Proof} (ZKP) systems to address these limitations. 
In particular, Shamsabadi \etal~\cite{shamsabadi2024confidential} introduced \emph{Confidential-DPproof}, which evaluates the correct execution of \emph{Differentially Private Stochastic Gradient Descent} (DP-SGD)~\cite{Abadi_2016}. 
This method enables the verification of a privacy level derived from training hyperparameters and offers an alternative to traditional black-box privacy audits. 
Recently, Abbaszadeh \etal~\cite{abbaszadeh2024zero} proposed \emph{Kaizen}, a ZKP training system relying on recursive composition of sumcheck-based proofs from each iteration to prove the entire training. 
For a VGG network, \emph{Kaizen} incurs a proof generation time of 10 minutes per epoch.

\paragraph{\textbf{Limitations}} 
The practical deployment of proof-based verifications within FL is significantly hindered by the computational cost of the proofs' generation. 
For instance, \emph{DP-Proof}~\cite{shamsabadi2024confidential} costs up to $100$ hours for a complete proof generation for a neural network trained on the \textsf{CIFAR-10} dataset.
In addition, ZKP-based PoL systems do not defend against data poisoning attacks~\cite{abbaszadeh2024zero} as an execution of the optimization circuit using poisoned data still yields a valid proof. 
Hence, the defensive range of these approaches against Byzantine behaviors is restricted to model poisoning attacks such as gradient inversion.

\paragraph{\textbf{Privacy-preserving proof-based verification}}

Recent works have adopted ZKP mechanisms to add privacy guarantees to the training process~\cite{shamsabadi2024confidential, abbaszadeh2024zero}. 
These systems ensure that the verifier, \ie, the FL server, learns only the validity of a worker's encrypted update, enabling informed decisions on whether to retain or discard it. 
In addition, by adapting the verified training circuit from SGD to DP-SGD~\cite{Abadi_2016}, Shamsabadi \etal \cite{shamsabadi2024confidential} demonstrated that these methods can provide a cryptographically-proven privacy budget, complementing secure aggregation encryption.  

\subsection{Property Inference attacks~(PIA)}
Property inference attacks~\cite{melis2018exploiting, wang2018inferringclassrepresentativesuserlevel, parisot2021propertyinferenceattacksconvolutional, mo2021layerwisecharacterizationlatentinformation, kerkouche2023clientspecificpropertyinferencesecure} are ML privacy breaches in which an adversary seeks to infer properties of the training data that are not necessarily part of the primary learning objective~\cite{parisot2021propertyinferenceattacksconvolutional, mo2021layerwisecharacterizationlatentinformation}.
For instance, Parisot \etal~\cite{parisot2021propertyinferenceattacksconvolutional} infer the proportion of training images containing people with an open mouth from a facial recognition model. 
Unlike membership or attribute inference attacks, which aim to extract specific information about individual data points, property inference attacks focus on deducing global properties from the training data.
PIAs often employ a white-box shadow model framework~\cite{parisot2021propertyinferenceattacksconvolutional}, in which the adversary has access to labeled datasets $\{\mathcal{D}_{p}\}$ (\ie, containing the property) and $\{\mathcal{D}_{\neg p}\}$ (\ie, without the property), to train shadow models. 
These models generate a meta-dataset in which shadow model parameters are paired with binary property labels indicating the presence or absence of the target property in the associated training data. 
Then, an inference model is trained on this meta-dataset, to infer the target property.
    
\paragraph{\textbf{PIA for ML auditing}} 
Duddu \etal~\cite{duddu2024attesting} have proposed an ML property attestation method to verify the distribution of training data, which combines PIAs and secure two-party computation to ensure privacy during evaluation.
Juarez \etal~\cite{juarez2022blackboxauditsgroupdistribution} relied on a black-box approach based on a combination of PIA and membership inference to detect the under-representativeness of demographic groups in the training set of the model and achieve up to 100\% AUC. 
Other works \cite{song2019auditingdataprovenancetextgeneration,liu2020forgottenmethodassessmachine, miao2021audioauditoruserlevelmembership} use membership inference principles to infer if a model's training set contains a specific user's data. 
These approaches relate to efforts to audit compliance with data privacy principles, such as the \emph{Right of Erasure} under GDPR.
\textbf{Building upon PIAs for auditing, our work, to the best of our knowledge, is the first application of these principles to Byzantine-resilient privacy-preserving FL.}  
 
\paragraph{\textbf{PIA in FL}}
Melis \etal~\cite{melis2018exploiting} have shown that property inference is feasible through gradient observation during FL training. 
More precisely, by analyzing per-iteration updates, they inferred properties such as gender or political views from user data. 
Mo \etal~\cite{mo2021layerwisecharacterizationlatentinformation} analyzed FL updates and showed different degrees of vulnerability to PIA in specific layers of the trained neural network.  
Finally, Kerkouche \etal~\cite{kerkouche2023clientspecificpropertyinferencesecure} have demonstrated that while secure aggregation protocols are designed to protect worker-level privacy, they are not immune to PIAs targeting the aggregated model to infer worker-specific sensitive properties.
Unlike standard privacy inference applications, \textbf{our approach uses the general PIA methodology to assess the reliability of worker updates in FL rather than conducting privacy attacks}. 
In our work, the \emph{``property''} reflects a specific misbehavior, which conceptually aligns with PIA, in which a property is an identifiable pattern embedded in the model's parameters that stems either from the data or the optimization algorithm.

\begin{table*}[t]
\centering
\begin{threeparttable}
    \caption{Defensive capabilities and cost of Byzantine resilience approaches in FL}
    \label{tab:related_work_detailed_summary_v3}
    \scriptsize
    \setlength{\tabcolsep}{3pt} 
    
    \begin{tabularx}{\textwidth}{@{} c l c c c c L l l @{}}
        \toprule
         & Method & Dishonnest Maj. & Non-IID & Data Poison & Model Poison & Privacy Aspect & Server Cost & Worker Cost \\
        \midrule
        
        \multirow{9}{*}{\rotatebox{90}{Outlier Exclusion}} 
        & Krum \cite{blanchard2017machine} & \redcross & \redcross & \textcolor{ForestGreen}{\checkmark} & \textcolor{ForestGreen}{\checkmark} & None & Median+Agg & $\mathsf{LocalUpdate}$ \\
        & Trimmed-Mean/CWTM \cite{yin2021byzantinerobustdistributedlearningoptimal} & \redcross & \redcross & \textcolor{ForestGreen}{\checkmark} & \textcolor{ForestGreen}{\checkmark} & None & Median+Agg & $\mathsf{LocalUpdate}$ \\
        & ClipCluster \cite{karimireddy2021learninghistorybyzantinerobust} & \redcross & \redcross & \textcolor{ForestGreen}{\checkmark} & \textcolor{ForestGreen}{\checkmark} & None & Median+Agg& $\mathsf{LocalUpdate}$ \\
        & FLTrust \cite{cao2022fltrustbyzantinerobustfederatedlearning} & \textcolor{ForestGreen}{\checkmark} & \redcross & \textcolor{ForestGreen}{\checkmark} & \textcolor{ForestGreen}{\checkmark} & None & Similarity+$\mathsf{Agg}$ & $\mathsf{LocalUpdate}$\\
        \cmidrule(lr){2-9}
                
        & Choffrut \cite{choffrut2023sablesecurebyzantinerobust} & \redcross & \redcross & \textcolor{ForestGreen}{\checkmark} & \textcolor{ForestGreen}{\checkmark} & FHE & FHE (Median+$\mathsf{Agg}$)  & $\mathsf{LocalUpdate}$+\enc \\
        & Xu \cite{MA2022103561} & \redcross & \redcross & \textcolor{ForestGreen}{\checkmark} & \textcolor{ForestGreen}{\checkmark} & AHE & AHE(Median+$\mathsf{Agg}$) & $\mathsf{LocalUpdate}$+\enc \\
        & Wang \cite{9685821} & \redcross & \redcross & \textcolor{ForestGreen}{\checkmark} & \textcolor{ForestGreen}{\checkmark} & FHE & FHE (Median+$\mathsf{Agg}$) & $\mathsf{LocalUpdate}$+\enc \\
        & Guerraoui \cite{guerraoui2021differentialprivacybyzantineresilience} & \redcross & \redcross & \textcolor{ForestGreen}{\checkmark} & \textcolor{ForestGreen}{\checkmark} & DP updates & Median+$\mathsf{Agg}$ & $\mathsf{LocalUpdate}_{\mathsf{DPSGD}}$ \\

        \midrule

        \multirow{3}{*}{\rotatebox{90}{Proof-based}} 
        & Jia PoL \cite{jia2021proof} & \textcolor{ForestGreen}{\checkmark} & \textcolor{ForestGreen}{\checkmark} & \redcross & \textcolor{ForestGreen}{\checkmark} & None & $\mathsf{Verif}$+$\mathsf{Agg}$ &  $\mathsf{LocalUpdate}+\mathsf{ProofGen}$\\
        & Shamsabadi \cite{shamsabadi2024confidential} & \textcolor{ForestGreen}{\checkmark} & \textcolor{ForestGreen}{\checkmark} & \redcross & \textcolor{ForestGreen}{\checkmark} & ZKP + Verif. DP $\epsilon$ & $\mathsf{Verif}$+$\mathsf{Agg}$ & $\mathsf{LocalUpdate}+\mathsf{ProofGen}$  \\
        & Abbaszadeh \cite{abbaszadeh2024zero} & \textcolor{ForestGreen}{\checkmark} & \textcolor{ForestGreen}{\checkmark} & \redcross & \textcolor{ForestGreen}{\checkmark} & ZKP & $\mathsf{Verif}$+$\mathsf{Agg}$ & $\mathsf{LocalUpdate}+\mathsf{ProofGen}$  \\
        & Garg \cite{garg_zkPoT} & \textcolor{ForestGreen}{\checkmark} & \textcolor{ForestGreen}{\checkmark} & \redcross & \textcolor{ForestGreen}{\checkmark} & ZKP via recursive GKR & $\mathsf{Verif}$+$\mathsf{Agg}$ & $\mathsf{LocalUpdate}+\mathsf{ProofGen}$  \\

        \midrule 
    
         & \textbf{Our approach} & \textcolor{ForestGreen}{\checkmark} \tnote{†} & \textcolor{ForestGreen}{\checkmark} \tnote{‡} & \textcolor{ForestGreen}{\checkmark} & \textcolor{ForestGreen}{\checkmark} & Threshold FHE+DP & FHE (Filter+$\mathsf{Agg}$.) & $\mathsf{LocalUpdate}$+\enc \\

        \bottomrule
    \end{tabularx}
    
    \begin{tablenotes}

        \footnotesize

        \item[†] Relies on the predictive performance of the meta-classifier, which can be effectively optimized (\cf{} Section \ref{sec:conv_analysis}).

        \item[‡] Builds on training the meta-classifier to handle distribution shifts, a capability that can be systematically incorporated into the training pipeline (\cf{} Section \ref{sec:filtering_data_gen}).

    \end{tablenotes}
    
\label{tab:comparison}
\end{threeparttable}
\end{table*}

\paragraph{\textbf{Overall positioning}}
Our approach improves on the first class of solutions by removing the honest-majority requirement and relying on a meta-classifier whose performance does not depend on the number of Byzantine workers. 
The maximum tolerable Byzantine ratio thus directly follows from the meta-classifier’s predictive accuracy (\cf{} Section~\ref{sec:conv_analysis}).
To remain effective under non-IID data, we train the meta-classifier on a wide range of shadow updates stemming from diverse local distributions (\cf{} Section~\ref{sec:filtering_data_gen}).
Our approach eliminates the need for a computationally heavy proof-generation operation while providing a broader defensive range, including data poisoning threats, which create detectable anomalies in the updates. 
Table \ref{tab:comparison} summarizes this position.

\section{Background}
\label{sec:background}
\paragraph{\textbf{Federated learning with model averaging}}
FL is a distributed learning framework in which different entities, including a server and multiple workers, collaborate to train a global shared model. 
Starting from an initial common model, each worker locally performs several SGD steps using its data samples and offloads the updated model's parameters to the aggregation server. 
This server generates a unified model by averaging the weights (\fedavg{}) of the workers' updates~\cite{mcmahan2023communicationefficientlearningdeepnetworks}. 
This process requires several rounds to achieve high predictive performance across all workers' data.
As shown in previous studies, several parameters impact the FL convergence speed, such as the number of workers involved~\cite{zhou2023parameter}, the heterogeneity of the workers' data~\cite{zhao_2018}, and the' dropouts of the workers~\cite{wen2022federated}.
Algorithm~\ref{alg:fedavg} describes this process.

\begin{algorithm}[ht!]
\caption{Federated Averaging (FedAvg) \cite{mcmahan2023communicationefficientlearningdeepnetworks}}
\label{alg:fedavg}
\begin{algorithmic}[1]
\State \textbf{Input:} Number of communication rounds $T$, number of workers $n$, local epochs $E$ and learning rate $\eta$. $\{\mathcal{D}_i\}_{i\in[n]}$ the set of workers datasets.
\State \textbf{Server initializes:} Global model $\theta^{\text{init}}$
\For{$t = 1, 2, \cdots, T$} 
    \State Broadcast global model $\theta_t$ to workers
    \For{each worker $k \in [n]$ \textbf{in parallel}}
        \State $\theta^{t+1}_k \gets \textsf{LocalUpdate}(\theta_t, \mathcal{D}_k, E, \eta)$
    \EndFor
    \State Aggregate: $\theta^{t+1} \gets \sum_{k \in [n]} \frac{n_k}{\sum_{i \in [n]}n_i} \theta^{t+1}_k$ 
\EndFor
\State \textbf{Output:} Final global model $\theta^T$
\end{algorithmic}
\end{algorithm}

\paragraph{\textbf{Fully Homomorphic Encryption (FHE)}}
FHE is a well-established privacy-preserving technique, with at least four main schemes being used in practice, namely BGV, BFV, CKKS and FHEW/TFHE~\cite{BFV2012,BGV2011,CKKS_2017,fhew_2014,tfhe_2018}. 
From a security perspective, they all rely on the Learning With Errors problem (LWE)~\cite{regev2024lattices} or its ring variant \cite{rlwe}. 
These schemes are framed as follows.
Let $\scheme=(\keygen,\enc,\dec,\eval)$ denote a public key homomorphic encryption scheme defined by the following set of probabilistic polynomial time (PPT) algorithms, given security parameter $\mathsf{\lambda}$:
\begin{itemize}[leftmargin=*]
\item[$\bullet$] $(\pk, \evk, \sk) \leftarrow \keygen(1^{\lambda})$: outputs an encryption key $\pk$, a public evaluation key $\evk$ and a secret decryption key $\sk$. 
The evaluation key is used during homomorphic operations while $\evk$ corresponds to a relinearization key for a leveled or somewhat scheme (such as BFV~\cite{BFV2012}) or to a bootstrapping key when a scheme such as TFHE or CKKS~\cite{tfhe_2018, CKKS_2017} is used.
\item[$\bullet$] $[m]_{\pk} \leftarrow \enc{(m,\pk)}$: turns a message $m$ into a ciphertext $c$ using the public key $\pk$.
\item[$\bullet$] $m \leftarrow \dec{([m]_{\pk},\sk)}$: transforms a message $c$ into a plaintext $m$ using the secret key $\sk$.
\item[$\bullet$] $c_f \leftarrow \eval{(f, c_1,\cdots, c_K, \evk)}$: evaluates the function $f$ on the encrypted inputs $(c_1, \dots, c_k)$ using the evaluation key $\evk$. 
In the general case in which non-exact FHE schemes are allowed, $\eval$ is such that with high enough ($\lambda$-independent) probability:
\begin{small}
\begin{equation}
\label{eq:approx:cond}
|\dec(\eval(f, [m_1]_{\pk} ,\cdots, [m_K]_{\pk})- f(m_1,\cdots,m_{K}))|\leq\varepsilon
\end{equation}
\end{small}
\end{itemize}
Throughout this paper, we rely on the CKKS scheme \cite{CKKS_2017}, which is described in detail in Appendix~\ref{sec:appendix_ckks}.
 
\section{Workers' Misbehaviors in FL}
\label{sec:relevant_properties_for_fl}
Distributing training across multiple workers leads to convergence challenges, which requires making the training process resilient to adversarial behaviors (\emph{i.e.}, Byzantine resilience). 
This shifts the trust model from assuming honest-but-curious participants to accounting for potentially malicious ones, as explored in the FL literature.
This broadly involves attacks aiming to disrupt training (\eg{}, preventing global gradient descent) or drive the model towards a poisoned path.  
Addressing these threats when worker updates are encrypted motivates our work.

\subsection{\textbf{Byzantine workers}}
\label{sec:Byzantine_behaviours_discussion}
Byzantine attacks can be categorized into two groups:  
\begin{enumerate}[leftmargin=*]
    \item{\textbf{Preventing convergence.}} Attacks designed to prevent convergence do so by crafting updates that decrease, rather than improve, the aggregated model's utility on the collective validation data. 
    This can be achieved by training the global model on a deliberately manipulated version of local data, \eg, through label flipping~\cite{jebreel2022defendinglabelflippingattackfederated}, or altering the training algorithm, \eg, performing gradient ascent instead of descent \cite{xu2022agicapproximategradientinversion, leite2024federatedlearningattackimproving}.
    \item{\textbf{Poisoned convergence.}} The deliberate introduction of targeted predictive errors is known as \emph{poisoning attacks} \cite{jebreel2022defendinglabelflippingattackfederated, 
    li2023learningbackdoorfederatedlearning, nguyen2023backdoorattacksdefensesfederated, leite2024federatedlearningattackimproving}, in which the Byzantine worker embeds a backdoor task within its training data. 
    For instance, in image data, this backdoor is often a subtle visual feature intended to trigger a predefined, incorrect prediction whenever it appears, hence dictating the model's behavior. 
    This is accomplished by crafting data samples that contain the backdoor and associating these samples with class labels that result in the intended misclassification~\cite{li2023learningbackdoorfederatedlearning, nguyen2023backdoorattacksdefensesfederated}. 
    Thus, unlike the previous category, these attacks do not seek to disrupt convergence but instead achieve convergence on both the main task and the backdoor one.
\end{enumerate}

\subsection{Considered Filtering Properties}

We investigate the following workers' misbehaviors:
\begin{itemize}[leftmargin=*]
    \item{\textbf{Poisoning via backdoor insertion.}}
    We investigate the ability of our approach to filter out updates trained on datasets containing triggers associated with a misclassification. 
    To do so, we embed a 4-pixel square on \GTSRB{} and deceitfully associate these samples with the sign \emph{``End of speed limit''}. 
    Figure \ref{fig:backdoor_illustration} illustrates examples of these samples.
    \item{\textbf{Poisoning via label-flipping and shuffling.}}
    We evaluate the effectiveness of our approach in detecting updates trained on datasets affected by binary label-flipping and, more generally, label-shuffling attacks.
    To simulate such threats, we flip the binary labels in the \acsincome{} datasets.
    This manipulation degrades the global model’s accuracy and can bias its predictions toward the adversary’s objectives.
    \item{\textbf{Gradient manipulation attacks.}}
    In this setting, Byzantine workers manipulate the gradient vector by setting it as the opposite of the genuine one, which contributes to maximizing the global objective.
\end{itemize}

\begin{figure}[ht!]
  \centering
  \begin{minipage}{0.32\columnwidth}
    \centering
    \includegraphics[width=\linewidth]{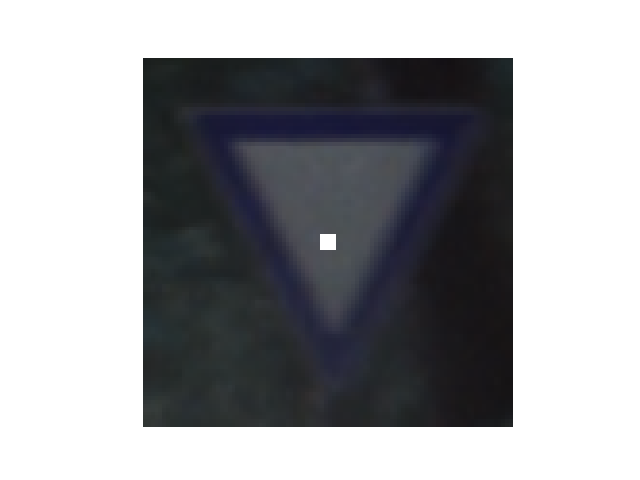}
  \end{minipage}\hfill
  \begin{minipage}{0.32\columnwidth}
    \centering
    \includegraphics[width=\linewidth]{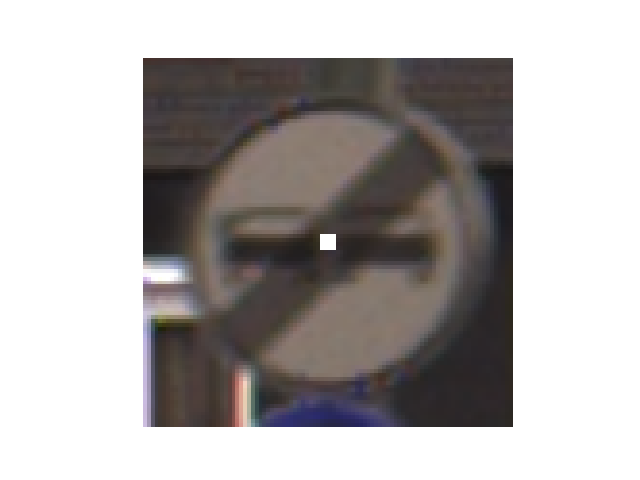}
  \end{minipage}\hfill
  \begin{minipage}{0.32\columnwidth}
    \centering
    \includegraphics[width=\linewidth]{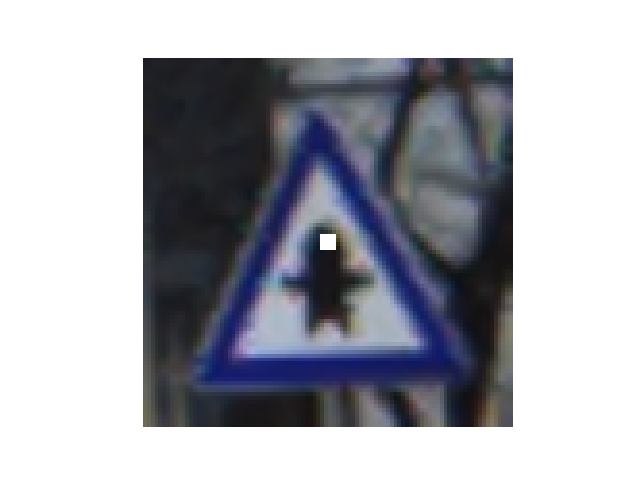}
  \end{minipage}

  \caption{Examples of \GTSRB{} samples embedded with the trigger and labeled as \emph{``End of speed limit''}.}
  \label{fig:backdoor_illustration}
\end{figure}

\section{Encrypted Property Filtering Framework}
\label{sec:our_approach}
This section presents a high-level description of our approach, without focusing on a particular misbehavior and without detailing the building blocks of the filtering process $\mathsf{Filter}_P$. 
The details of the considered attacks, the filters, and the homomorphic property inference are respectively provided in Sections \ref{sec:relevant_properties_for_fl}, \ref{sec:filtering_data_gen}, \ref{sec:fhe_efficient_filters} and \ref{sec:ckks_building_blocks_for_filtering}.
Hereafter, we denote the encrypted model of a worker $j$ at FL training iteration $t$ by $\llbracket \theta_j^t \rrbracket $\footnote{In practice, $\llbracket \theta_j^t \rrbracket$ corresponds to many ciphertexts. 
Indeed, given the FHE scheme parameters, deep learning models' updates $\theta_j^t$ usually require several ciphertexts.}. 
Similarly, $\llbracket \theta^{t,s}_j \rrbracket$ denotes the encryption of the weights of layer $s$ in worker $j$’s model at iteration $t$. 
In addition, we denote by $\llbracket \theta_j^t \rrbracket \cdot \llbracket x \rrbracket$ the homomorphic multiplication of all the ciphertexts of an update $\theta_j^t$ by the ciphertext $\llbracket x \rrbracket$.

\textbf{Threat model.}  
We consider an FL setting in which $n$ workers collaborate to train a global model with an honest-but-curious server that may collude with a subset of workers for privacy inference purposes only (\ie{}, malicious behavior is not considered for the server). 
We also require a bound $B_{\text{coll}}$ on the number of colluding workers, which is used to set up the collaborative decryption threshold $t$.
We also assume at least one honest worker $(n_h \geq 1)$, thus limiting the number of Byzantine attackers to $n_b \leq n-1$.

\textbf{FL with FHE property filtering.}
Before training, workers initialize a threshold CKKS scheme with parameter $\tau > B_{\text{coll}}$, each holding a secret key share $\mathsf{sk}_i$ and the public key $\pk$. 
The evaluation and public keys $(\evk, \pk)$ are sent to the server with the collaborative decryption requiring at least $\tau$ shares.
The threshold CKKS setup is detailed in Appendix~\ref{sec:appendix_thCKKS}.

When the FL training begins ($t=0$), the server broadcasts the initial plaintext model. 
Afterwards, workers compute local updates $\mathsf{LocalUpdate(\cdot)}$, encrypt them under $\pk$ and transmit the ciphertexts to the server.
Then, the server computes the encrypted difference $\llbracket \Delta^t_i \rrbracket = \llbracket \theta^{t, s} \rrbracket - \llbracket \theta^{t, s}_i \rrbracket$ in which $s$ is the identifier of the layer with the highest separability (\cf{}, Section \ref{sec:filtering_data_gen}).  
An FHE property inference process $\mathsf{Filter}_P(\cdot)$ on $\llbracket \Delta^t_i \rrbracket$ results in an encrypted filter value $\llbracket P_i^t \rrbracket$ in which $P_i^t \in [0, 1]$, is the filter value for worker $i$. 
Each update is homomorphically scaled by its corresponding filter value, and the scaled updates are aggregated (via averaging) to form the encrypted non-normalized model $\llbracket\theta_{\text{scaled}}^{t+1}\rrbracket$. 
The server homomorphically computes $\llbracket S^t \rrbracket = \sum_{i=1}^{n} \frac{n_i}{\sum_{i=1}^{n} n_i} \llbracket P^t_i \rrbracket$ and its inverse $\llbracket S_{\text{inv}}^t \rrbracket$ to normalize the global model. 
Hence, producing  $\llbracket\theta^{t+1}\rrbracket$, as outlined in Algorithm \ref{alg:newton_raphson_fhe}. 
Workers collaboratively decrypt the aggregated model using their private keys shares $\mathsf{sk}_i$. 
Concurrently, the server bootstraps the aggregated model's ciphertext $s$ to compute the next round's $\Delta^{t+1}_i$. 

\begin{algorithm}[ht!]
\caption{FL with encrypted filtering}
\label{alg:approach1}
\begin{algorithmic}[1]
\State \textbf{Setup:} $(\pk, \{\sk_i\}^n_{i=1}, \evk)\leftarrow \ThKeyGen(1^{\lambda}, n, \tau)$ with $\tau > B_{\text{coll}}$. Server gets $\evk$. 

\For{each round $t=1 \dots T$}
    \Statex \textbf{Server:}
    \State Receive $\llbracket \thetaJt \rrbracket$ from workers $j\in \{1, \dots, n\}$
    \State Compute difference $\llbracket \DeltaJt \rrbracket \leftarrow \llbracket \theta^{t, s}_{\text{fresh}} \rrbracket - \llbracket \theta^{t, s}_j \rrbracket$
    \State Apply filter $\llbracket \PJt \rrbracket \leftarrow \FilterP(\llbracket \DeltaJt \rrbracket)$
    \State Scale all ciphertexts  $\llbracket \thetaPrimeJ \rrbracket \leftarrow \llbracket \PJt \rrbracket \cdot \llbracket \thetaJt \rrbracket$
    \State \fedavg{} $\llbracket \thetaTwoPrimeTplusOne_{\text{scaled}} \rrbracket \leftarrow \sum_{j=1}^n \left(\frac{n_j}{ \sum^{n}_{k=1}n_k}\right) \cdot \llbracket \thetaPrimeJ \rrbracket$
    \State Compute sum $\llbracket \STplusOne \rrbracket \leftarrow \sum_{j=1}^n \left(\frac{n_j}{ \sum^{n}_{k=1}n_k}\right) \cdot \llbracket \PJt \rrbracket$
    \State Compute inverse $\llbracket \SinvTplusOne \rrbracket \leftarrow \EncInv(\llbracket \STplusOne \rrbracket)$
    \State Normalize $\llbracket \thetaTplusOne \rrbracket \leftarrow \llbracket \thetaTwoPrimeTplusOne_{\text{scaled}} \rrbracket \cdot \llbracket \SinvTplusOne \rrbracket $
    \State (\textit{Offline}) Bootstrap $\llbracket \theta_{\text{fresh}}^{t+1, s} \rrbracket \leftarrow \Boot( \llbracket \theta^{t+1, s})$
    \Statex
    \Statex \textbf{Workers ($j$):}
    \State Receive $\llbracket \thetaTplusOne \rrbracket$
    \State Decrypt collaboratively $\rightarrow \thetaTplusOne$
    \State Local update $\thetaNewJ \leftarrow \LocalUpdate(\cD_j, \thetaTplusOne)$
    \State Encrypt $\rightarrow \llbracket \thetaNewJ \rrbracket$
    \State Send $\llbracket \thetaNewJ \rrbracket$ to Server (for round $t+1$)
\EndFor
\Statex
\State \textbf{Output:} Final model $\theta^T$
\end{algorithmic}
\end{algorithm}

\subsection{Convergence Analysis}
\label{sec:conv_analysis}

\subsubsection*{Model and setup.}
We consider a federated learning system with $n$ workers, where $n_b$ is the number of Byzantine workers.  

Let $q = n_b/n$.  We define The global objective over the honest workers as follows:
\[
F(\theta) = \sum_{i \in H} \left( \frac{n_i}{\sum_{k \in H} n_k} \right) \cdot F_i(\theta),
\]
in which $H$ is the set of honest workers.

\begin{definition}[Ideal filtering]
We define the ideal filter as the one that assigns:
\[
P_i = 1 \ \forall i \in H, \quad P_j = 0 \ \forall j \in B.
\]
Where $\{ P_i \}$ are the filter values associated with each worker's update.
\end{definition}

\begin{definition}[Real filtering] 
We refer to the real filtering as the one obtained from a non-ideal setting, \ie{} in practice, with a filter defined as:
\begin{align*}
\mu_h &= \mathbb{E}[P_i \mid i \in H], \quad \sigma_h^2 = \mathbb{V}[P_i \mid i \in H]. \\
\mu_b &= \mathbb{E}[P_j \mid j \in B], \quad \sigma_b^2 = \mathbb{V}[P_j \mid j \in B]
\end{align*}
\end{definition}

\textbf{Geometric intuition.}
Figure~\ref{fig:geometric_intuition} provides a geometric illustration of the filtering error at communication round~$t$.
In the ideal scenario, either no Byzantine workers are present ($q=0$) or the meta-classifier matches the ideal filter performance, which results in the aggregate $\theta^{t+1}_{\text{ideal}}$ following the standard \fedavg{} trajectory~\cite{li2020convergence}, converging towards a neighborhood of the global optimum.
Under non-ideal conditions ($q>0$ and imperfect filtering), the server obtains a perturbed aggregate $\theta^{t+1}_{\mathrm{real}}$, which deviates from the ideal path by a filtering-induced displacement.
This displacement, represented by the vector $|\theta^{t+1}_{\mathrm{real}}-\theta^{t+1}_{\mathrm{ideal}}|$, captures the geometric effect of stochastic filtering errors.
Our convergence analysis formalizes this intuition by bounding the magnitude of this shift, thereby quantifying how imperfect filtering distorts the descent trajectory of standard \fedavg{}.
\begin{figure}[ht!]
    \centering
    \includegraphics[width=0.9\linewidth]{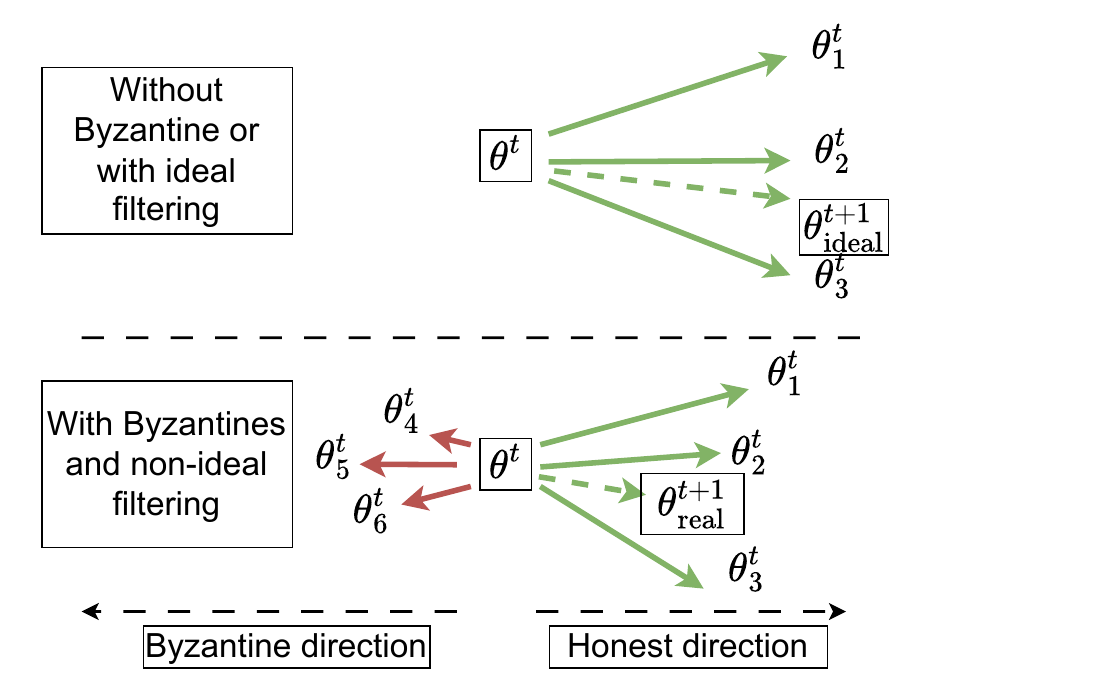}
    \caption{Illustration of the effect our filtering approach in a 2-dimensional parameter space with 6 workers and 3 Byzantines.}
    \label{fig:geometric_intuition}
\end{figure}

\subsubsection*{Assumptions.}
We adopt the following standard assumptions:

\begin{enumerate}[leftmargin=*]
    \item $F$ is $L$-smooth and $\mu$-strongly convex.
    
    \item \textbf{Bounded honest gradients:} Honest worker updates are uniformly bounded by a constant $G$:
    \[
    \|\Delta^t_i\| \leq G \quad \text{for } i \in H.
    \]
    \item \textbf{Bounded Byzantine gradients:} Byzantine worker updates are uniformly bounded by a constant $D_{\text{byz}}$:
    \[
    \|\Delta_j^t\| \leq D_{\text{byz}} \quad \text{for } j \in B.
    \]
    This assumption can be achieved via server-side clipping \cite{zhang2021understandingclippingfederatedlearning}.
    \item \textbf{Bounded global gradient:} The global gradient is bounded by a constant $C_g$:
    \[
    \norm{\nabla F(\theta^t)} \le C_g
    \]
    
    \item \textbf{Base error:} The term $\mathbf{B}_t$ captures all non-IID worker drift and local SGD variance (\ie, all non-Byzantine error):
    \[
    \mathbf{B}_t := (\theta^{t+1}_{\mathrm{ideal}} - \theta_{\mathrm{ideal}}^t) - \eta\nabla F(\theta^t)
    \]
    
    \item \textbf{Bounded base error:} The base error term $\mathbf{B}_t$  is bounded. These bounds are discussed in detail in \cite{li2020convergence}:
    \begin{itemize}
        \item $\E[\norm{\mathbf{B}_t} \mid \theta^t] \le E_{\mathrm{base}}$
        \item $\E[\norm{\mathbf{B}_t}^2 \mid \theta^t] \le E_{\mathrm{base}}^{(2)}$
    \end{itemize}
\end{enumerate}

\subsubsection*{Main results}
We start by characterizing the expected drift between the ideal aggregate and the real one via upper bounds on its expectation (Lemma \ref{lemma:expected_error}) and variance (Lemma \ref{lemma:variance_error}). 
In the following, we present our main convergence theorem, from which we derive a practical corollary that determines the minimal filter quality parameters $(\mu_h, \mu_b, \sigma^2_h, \sigma^2_b)$ required for a given Byzantine fraction $q$.
We provide the proofs of the following lemmas and the theorem in Appendix \ref{sec:appendix_conv_analysis}.

\begin{lemma}[Expected Filtering Error]
\label{lemma:expected_error}
At round $t$, let $\theta^{t+1}_{\mathrm{ideal}}$ denote the ideal aggregated update, and $\theta^{t+1}_{\mathrm{real}}$ the actual update. 
The filtering noise
$\delta_t := \mid \theta^{t+1}_{\mathrm{real}} - \theta^{t+1}_{\mathrm{ideal}} \mid$
satisfies:
\[
\mathbb{E}[\|\delta_t\| \mid \theta^t] \le E_{\mathrm{filter}},
\]
in which:
\begin{align*}
E_{\text{filter}} = & n_h G \left( |1 - \mu_h| + \sigma_h \right)  + n_b D_{\text{byz}} \left( \mu_b + \sigma_b \right)
\end{align*}
\end{lemma}

\begin{lemma}[Variance of the filtering error]
\label{lemma:variance_error}
Under assumptions (1-5), the second moment of the filtering noise satisfies
\[
\mathbb{E}[\|\delta_t\|^2 \mid \theta^t] \le E_{\mathrm{filter}}^{(2)},
\]
in which:
\[
E_{\mathrm{filter}}^{(2)} =
2n_h^2\bigl[(1-\mu_h)^2 + \sigma_h^2\bigr]G^2
+ 2n_b^2(\mu_b^2 + \sigma_b^2)D_{\text{byz}}^2
\]
\end{lemma}

\begin{theorem}[Convergence under filtering and base errors]
Let the assumptions above hold. 
The iterate $\theta^t$ of the filtered FedAvg satisfies:
\begin{align*}
\E[F(\theta^{T})] - F(\theta^{*}) &\le (1-\eta\mu)^{T}(F(\theta^{\text{init}})-F(\theta^{*})) \\
& \quad + \frac{1}{\eta\mu}\Big( C_{g}(E_{\mathrm{base}} + E_{filter}) \\
& \qquad \qquad + 2L(E_{\mathrm{base}}^{(2)} + E_{filter}^{(2)}) \Big)
\end{align*}
\end{theorem}

\begin{corollary}[Byzantine proportion and filter quality]
\label{cor:tolerable}
Let $\epsilon$ be the target suboptimality level, and $\epsilon_{\text{base}}$ be the limiting error induced by honest imperfections.
Assume that the filtering mechanism yields Byzantine filters with
$(\mu_b, \sigma_b)$ and let $q$ denote the proportion of Byzantine workers.
Define the final Byzantine cost
\begin{align*}
C_{\text{Byz}}(q,\mu_b,\sigma_b)
\;=\;
C_g\!\left(q\,\mu_b D_{\text{byz}}
+\sqrt{q}\,\sigma_b D_{\text{byz}}\right)
\\+4L\,n^2 q^2 (\mu_b^2+\sigma_b^2)D_{\text{byz}}^2
\end{align*}

Then, to guarantee convergence within $\epsilon$ of the optimum, it suffices that:
\begin{equation}
\label{eq:conv_requirement}
\frac{C_{\text{Byz}}(q,\mu_b,\sigma_b)}{\eta \mu}
\le
\epsilon - \epsilon_{\text{base}}
\end{equation}

Equivalently, for a fixed $(\mu_b,\sigma_b)$,
the maximal tolerable attacker fraction $q_{\max}$ satisfies
\begin{equation}
\label{eq:qmax}
q_{\max}
\;\approx\;
\frac{
\eta \mu (\epsilon - \epsilon_{\text{base}})
}{
C_g n D_{\text{byz}}(\mu_b + \sigma_b)
}
\end{equation}
\end{corollary}

\subsubsection*{Practical considerations}
Equations \eqref{eq:conv_requirement} and \eqref{eq:qmax} quantify how the filter's Byzantine mean $\mu_b$ and variance $\sigma_b$ trade-off against the tolerated attacker ratio $q$. 
As $q$ increases, the region $(\mu_b,\sigma_b)$ that provides convergence guarantees given system parameters, linearly shrinks toward the ideal filter $(0,0)$. 
Figure \ref{fig:q_max_feasability} illustrates this effect.

\begin{figure}[ht!]
    \centering
    \includegraphics[width=1.0\linewidth]{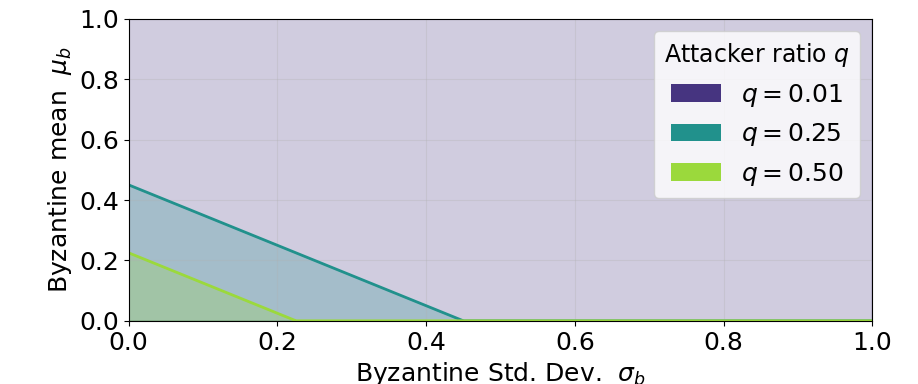}
    \caption{Feasibility regions for different Byzantine proportions according to $(\mu_b, \sigma_b)$ and fixed system parameters ($\epsilon, \epsilon_{base}, \eta, \mu, L, C_g, D_{byz}$). As $q$ grows, the space of filters that provide convergence guarantees to $\epsilon$ linearly shrinks towards the ($\mu_b = 0, \sigma_b=0 $) filter.}
    \label{fig:q_max_feasability}
\end{figure}

\section{Filtering Data Generation}
\label{sec:filtering_data_gen}
Our approach relies on an ML-based filtering mechanism on the server side to reweigh honest and adversarial updates. 
This filter operates on the update's differences with the previously aggregated model, denoted $\Delta_i^t$. 
These deltas carry significant inference potential and have been shown to enable up to reconstruction attacks in \fedavg{}-based federated learning~\cite{dimitrov2022dataleakagefederatedaveraging}.
In our case, we leverage them in a positive manner to detect Byzantine behaviors.

To collect statistically representative updates that capture both honest worker behavior and Byzantine manipulations, we perform several SGD training for $T$ epochs with different random initializations and data partitions. 
At each epoch $t \in [T]$, we introduce a forked Byzantine update $\bar{\theta}^{t+1}$ alongside the honest one $\theta^{t+1}$. 
The corresponding update differences are then computed as:
$\Delta_t = \theta^{t+1} - \theta^t$ and $\bar{\Delta}_t = \bar{\theta}^{t+1} - \theta^t$.
These update differences are labeled accordingly and stored in a dataset as shown in Algorithm~\ref{alg:filtering_data_generation}.

\begin{algorithm}[ht!]
\caption{Filtering data generation for a property $P$}
\label{alg:filtering_data_generation}
\begin{algorithmic}[1]
\Require A base dataset $D_{\text{shadow}}$, number of epochs $T$
\Ensure Dataset $\mathcal{D}_P = \{(\Delta_i, y_i)\}^{K}_{i=1}$ for a property $P$
\State $\mathcal{D}_P \gets \emptyset$ \Comment{Initialize dataset}

\For{$i = 1$ to $R$} 
\State $\theta^{\text{init}} \leftarrow \mathsf{RandomInit(Seed_i)}$
\State $D^i_{\text{shadow}} \leftarrow \mathsf{DirichletSample}(D_{\text{shadow}}, \alpha=1.0)$
\For{$t = 1$ to $T$}
    \State $\theta^{t+1} \leftarrow \mathsf{HonestUpdate}(\theta^t, D^i_{\text{shadow}})$
    \State $\bar{\theta}^{t+1} \leftarrow \mathsf{DishonestUpdate}_P(\theta^t, D^i_{\text{shadow}})$
    \State $\Delta_t = \theta^{t+1} - \theta^t$ \Comment{Honest update difference}
    \State  $\bar{\Delta}_t = \bar{\theta}^{t+1} - \theta^t$ \Comment{Byzantine update difference}
    \State $\mathcal{D}_P \leftarrow \mathcal{D}_P \cup (\Delta_t, 1)$   \Comment{Label: 1}
    \State $\mathcal{D}_P \leftarrow \mathcal{D}_P \cup (\bar{\Delta}_t, 0)$ \Comment{Label: 0}
\EndFor
\EndFor
\State \Return $\mathcal{D}_P$
\end{algorithmic}
\end{algorithm}

Our key idea is to explore diverse convergence trajectories in the parameter space of a given model architecture, induced by the different data distributions generated via $\mathsf{DirichletSample}(\cdot)$. 
These trajectories capture the statistical patterns arising from both $\mathsf{HonestUpdate}(\cdot)$ and $\mathsf{DishonestUpdate}_P(\cdot)$, the latter being a misbehavior in FL. 
For instance, $\mathsf{DishonestUpdate}_P(\theta^t, D)$ could correspond to performing a gradient-ascent step on $\theta^t$.

\subsubsection*{Filtering model type}
Given the data generated in Algorithm \ref{alg:filtering_data_generation}, SVMs are a natural choice for classifying honest/Byzantine updates for the following reasons:

\begin{itemize}[leftmargin=*]
\item \textbf{Robustness to high-dimensional data.} The update differences ($\Delta_t$) lie in a high-dimensional parameter space. 
SVMs efficiently handle high-dimensional data and can find an optimal separating hyperplane between genuine and adversarial updates~\cite{hsu2022proliferationsupportvectorshigh, NAKAYAMA201788}.
\item \textbf{Effective for small and structured datasets.} 
Given the computational cost of collecting update differences, our model must be able to learn statistical patterns from limited data. 
SVMs are known to perform well on \emph{horizontal data} \cite{qiao2013flexiblehighdimensionalclassificationmachines}.
\end{itemize}

\subsubsection*{Data for shadow update generation}
\label{sec:shadow_data_considerations}

We outline two practical approaches to obtain or construct a suitable $D_{\text{shadow}}$ for \textit{offline} filter training:
\begin{enumerate}[leftmargin=*]
    \item \textbf{Publicly available datasets.} The server can leverage public datasets analogous to the FL task's domain and model architecture. 
    The meta-classifiers are trained to identify differential properties and anomalies in ($\bar{\Delta}_i^t$) that stem from specific behaviors (\eg{}, backdoor injection, gradient ascent) rather than from features of the raw worker data itself. 
    These behavioral signatures in the update space are learnable even if $D_{\text{shadow}}$ is not identically distributed on the workers' data.

    \item \textbf{Synthetic data generation.} Generative models (GANs, VAEs) are commonly used in privacy attacks to generate shadow models' training data \cite{shokri2017membership,  guépin2023syntheticneedremovingauxiliary, li2024darkfeddatafreebackdoorattack, bendoukha2025} from smaller, statistically restricted datasets. 
    In our setting, this can improve the filtering mechanism by benefiting from more varied shadow updates, even from limited initial seed data, thus improving the robustness and coverage of the trained SVM filters.
\end{enumerate}

\noindent The core principle is that the anomalies introduced by these malicious behaviors into the update vectors ($\bar{\Delta}^t_i$) are often orders of magnitude more pronounced or qualitatively different from the variations observed between honest updates ($\Delta_i^t$) arising from differing local data distributions. 
To support these claims, our experimental setting integrates these considerations. 
Specifically, we use $\mathsf{EMNIST}$ and \acsincome{} California state datasets as $D_{\text{shadow}}$ and, respectively, the broader \femnist{} and 10 other \acsincome{} datasets from states (excluding California) for FL. 
Hence, representing the use of publicly available datasets scenario.  

\section{FHE-Efficient Filters}
\label{sec:fhe_efficient_filters}

Predictive models' inference on FHE encrypted data is computationally intensive~\cite{cryptoeprint:2023/257, stoian2023deepneuralnetworksencrypted, cryptoeprint:2023/647, 8861379}. 
Therefore, designing FHE-friendly models demands careful optimization to balance performance and accuracy. 
Due to their inherent simplicity, SVMs present a significant advantage for efficient inference on FHE-encrypted data. 
An SVM consists of support vectors $\{x_i\}_{i \in [k]}$, their coefficients and labels defining the decision boundary. 
Inference on an input $x$ is:
$$\hat{y} = \mathsf{Sign} \left( \left( \sum_{i=0}^{k-1} \alpha_i y_i \mathsf{Kern}( x_i, x) \right) + \beta \right)$$,
in which $\mathsf{Kern(\cdot, \cdot)}$ is a kernel function,  $\alpha_i$ is the Lagrange multiplier, $y_i$ is the label associated with the support vector $x_i$ and $\beta$ is a bias intercept value.
Hereafter, we investigate the use of SVM for encrypted filtering.

\subsection{{SVM inference on FHE encrypted data}}

For linearly or polynomially separable data, $\mathsf{Kern(\cdot , \cdot)}$ is a vector dot product or a low-degree exponentiation. 
For more complex decision boundaries, the kernel function is either a Radial Basis Function ($\mathsf{RBF}(x)=e^{-\gamma||x-x_i||^2}$), or a $\mathsf{Sigmoid}$. 
Consequently, SVM inference on FHE-encrypted data faces two computational bottlenecks:
\begin{enumerate}[leftmargin=*]
    \item The first is the nature of the kernel function. 
    While linear and low-degree polynomial kernels are suitable for FHE evaluation, $\mathsf{Sigmoid}$ and $\mathsf{RBF}$ kernels require high-degree polynomial approximations. 
    This induces significant computational overhead and cumulative approximation errors arising from summation over all support vectors, leading to incorrect predictions, especially for samples near the decision hyperplane. 
    \item When the kernel is linear or a low-degree polynomial, another computational bottleneck emerges due to the complexity of the SVM, which manifests in a large number of support vectors. 
    As a result, a significant number of homomorphic dot products and/or exponentiations are required, further increasing the computational burden.
\end{enumerate}

It is noteworthy that enforcing a linear kernel with non-linearly separable data leads to overfitting and poor performance~\cite{sandjib2023_linear_svms}. 
To avoid this, we explore low-degree polynomial kernels, which can reduce the support vector count at the cost of increased kernel degree. 
Balancing kernel degree, support vectors and SVM accuracy is crucial for FHE efficiency. 

Our approach identifies linearly or polynomially separable update components for each attack and performs a grid search to identify optimal SVM configurations for homomorphic inference with CKKS. 
In the following, we detail our strategy. 

 \subsection{Optimizing the SVM Filtering}

In PIAs, the input to the inference model is either the vector of flattened parameters from the target model~\cite{parisot2021propertyinferenceattacksconvolutional} or a pre-computed embedding of these parameters, obtained through techniques such as auto-encoders~\cite{Wang2020PropertyIA} or Principal Component Analysis (PCA)~\cite{parisot2021propertyinferenceattacksconvolutional}.

To address the challenge of high dimensionality, we analyze the target model's internal structure to identify components that have the strongest correlation with the property considered. 
Instead of providing the full update $\Delta^t_j$ to the SVM filtering model, we focus on specific layers in which the presence or absence of the targeted property is most reflected. 
By restricting encrypted inference to a single layer of each worker's updates, we significantly reduce the complexity of the SVM filter while maintaining and even enhancing the predictive performance. 

First, we identify the layer that captures the optimal property signal before reducing its dimensionality using Supervised Principal Component Analysis (SPCA)~\cite{BARSHAN20111357}. 
To do so, we conduct a grid search across all model layers and their potential dimensionality reductions while evaluating the resulting SVMs' predictive performance and expected FHE cost ($\mathsf{dim}$ and $n_{\mathsf{sv}}$).
Unlike standard PCA, which preserves only the directions of maximum variance, Supervised PCA (SPCA) considers label information by selecting components that are not only high-variance but also highly correlated with the target labels.
By aligning the projection directions with label-dependent structure, SPCA enhances class separability in the reduced space, leading to a more discriminative projection.

To showcase our approach, we consider four models for \femnist{}, \cifar{}, \acsincome{} and \GTSRB{} presented in Table \ref{tab:nn_models}. Figure~\ref{fig:pca_spca_projections} depicts the enhanced separability brought by SPCA compared to PCA on backdoor update differences from the first convolution layer. 
Indeed, SPCA widens the separation along the first component between honest and backdoor update clusters, 
moving from an average distance of around $2$ to more than $200$. 
This separation directly impacts the SVM's decision boundary and reduces the number of support vectors, thereby improving inference runtime.

\begin{table}[ht!]
\centering
\caption{Datasets, model architectures and attack types.}
\begin{tabular}{l|l|l}
\hline
\textbf{Dataset} & \textbf{Model Architecture}              & \textbf{Attack Type}     \\
\hline
FEMNIST          & LeNet-5 \cite{LeNet} & Label shuffling                \\
CIFAR-10         & ResNet-14 \cite{resnet}                      & Gradient Ascent          \\
GTSRB            & VGG11 \cite{vgg}                              & Backdoor                 \\
ACSIncome       & 3-layer MLP                       & Label Flipping           \\
\hline
\end{tabular}
\label{tab:nn_models}
\end{table}

\begin{figure}[ht!]
    \centering
    \captionbox*{\small{(a) Unsupervised PCA (53 support vectors)}}{\includegraphics[width=0.23\textwidth]{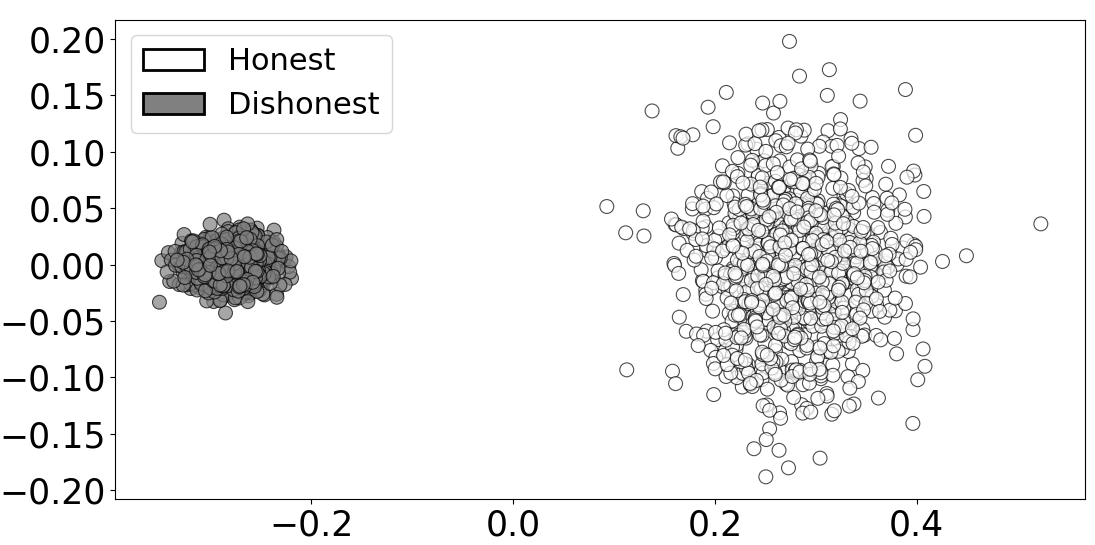}}
    \hfill 
    \captionbox*{\small{(b) Supervised PCA (38 support vectors)}}{\includegraphics[width=0.23\textwidth]{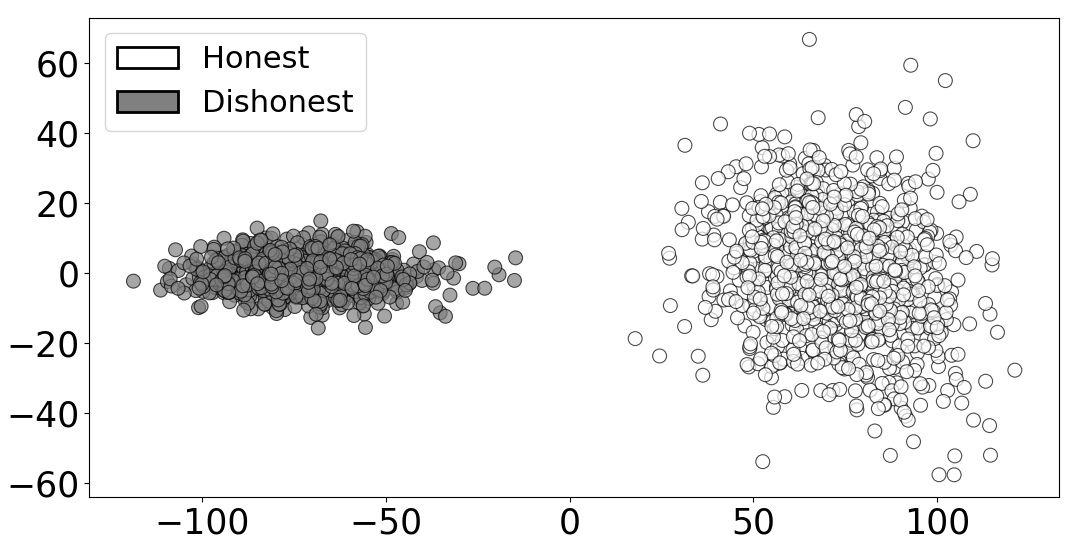}}
    \caption{PCA and SPCA projection of the input layer weights of updates $(\Delta_i)$. 
    The first component is the x-axis while the second component is the y-axis.}
\label{fig:pca_spca_projections}
\end{figure}

\subsection{Optimal filters via grid search}

We explore each attack's separability using SVMs with varying kernel degrees (linear, quadratic, cubic and quartic) applied to different SPCA-projected representations (8 and 64 dimensions) of layer weights.
To do so, we introduce $\mathcal{C}_{\mathsf{SVM}}^{\mathsf{FHE}}$ a score that balances the following criteria: 
\begin{enumerate}[leftmargin=*]
    \item The \textbf{predictive performance} of the SVM is represented by the $\mathsf{F1}$-score.
    \item The \textbf{model efficiency} is captured by the number of support vectors $(n_{\mathsf{sv}})$, which highly affects computational costs during inference.
    \item The \textbf{kernel simplicity}, denoted by $(\mathsf{deg})$, penalizes higher kernel degrees to reduce FHE inference cost.
    \item The \textbf{projection dimensionality}, referred to as $(\mathsf{dim})$, favors lower SPCA dimensions as they reduce the computational cost of the FHE projection.
     \item The \textbf{factor} ($\mathsf{P}$) denotes the extent of parallelization capabilities, which mainly applies to $(n_{\mathsf{sv}})$. 
\end{enumerate}
Thus, we define the following heuristic FHE-SVM score: 
$$
\mathcal{C}_{\mathsf{SVM}}^{\mathsf{FHE}}(\mathsf{F1}, \mathsf{dim}, \mathsf{deg}, n_{\mathsf{sv}}) = \frac{\mathsf{F1}\text{-score}}{\mathsf{dim} \cdot \mathsf{deg} \cdot (\mathsf{P} \cdot n_{\mathsf{sv}})}
$$

The choice of the $\mathsf{F1}$-score as the main evaluation metric is justified by the nature of the application. 
In particular, the $\mathsf{F1}$-score is preferred in ML-based anomaly detection applications~\cite{Pekar_2024} due to the asymmetric cost of misclassification.
Indeed, false negatives (\ie{}, honest updates classified as Byzantine) exclude valid contributions, leading to model degradation, while false positives allow malicious behavior to persist, compromising convergence. 

Since the evaluation of the sum over all support vectors $(\mathsf{SVs})$ during inference is highly parallelizable, the factor $\mathsf{P}$ adjusts the influence of $(n_{\mathsf{sv}})$ count on the overall score. 
This ensures that configurations with a higher number of support vectors are not unfairly penalized when parallel processing is achievable. 
The weight $\mathsf{P}$ ranges from $1$, indicating no parallelization capabilities, to $\frac{1}{n_{\mathsf{sv}}}$, representing maximum parallelization, in which each support vector is evaluated simultaneously, with intermediate values capturing varying parallelization capabilities.

Figure \ref{fig:bubble_plot_grid_search} shows the exploration of optimal SVM configurations with respect to $\mathcal{C}^{\mathsf{FHE}}_{\mathsf{SVM}}$ over the optimal layer for all models and datasets considered.
Table \ref{tab:best_settings} provides optimal settings for each property and the associated predictive performances.

\begin{figure*}[ht!]
    \centering
    \captionbox*{\scriptsize{Backdoor (\GTSRB{} and VGG11)}}{\includegraphics[width=0.245\textwidth]{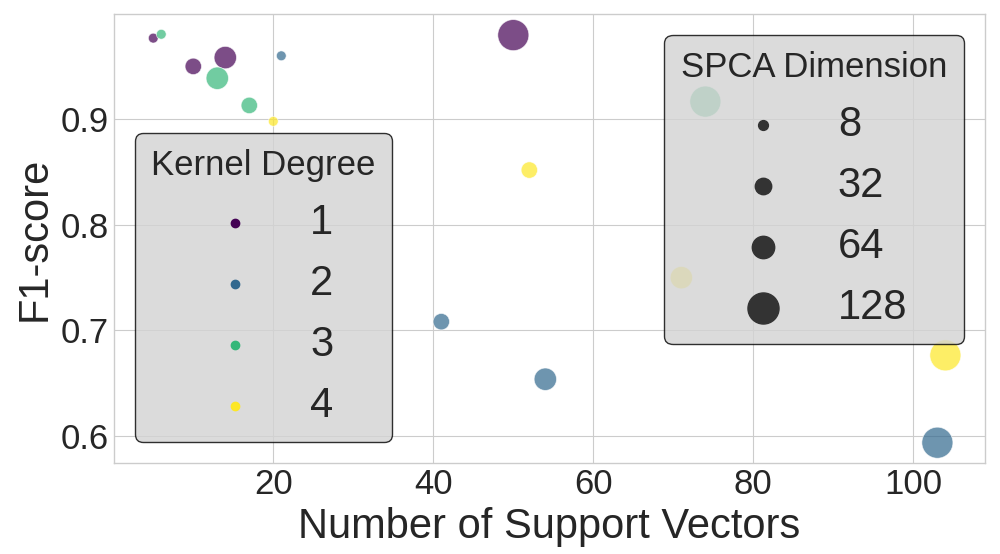}}
    \hfill
    \captionbox*{ \scriptsize{Gradient-ascent (\cifar{} and ResNet)}   }{\includegraphics[width=0.245\textwidth]{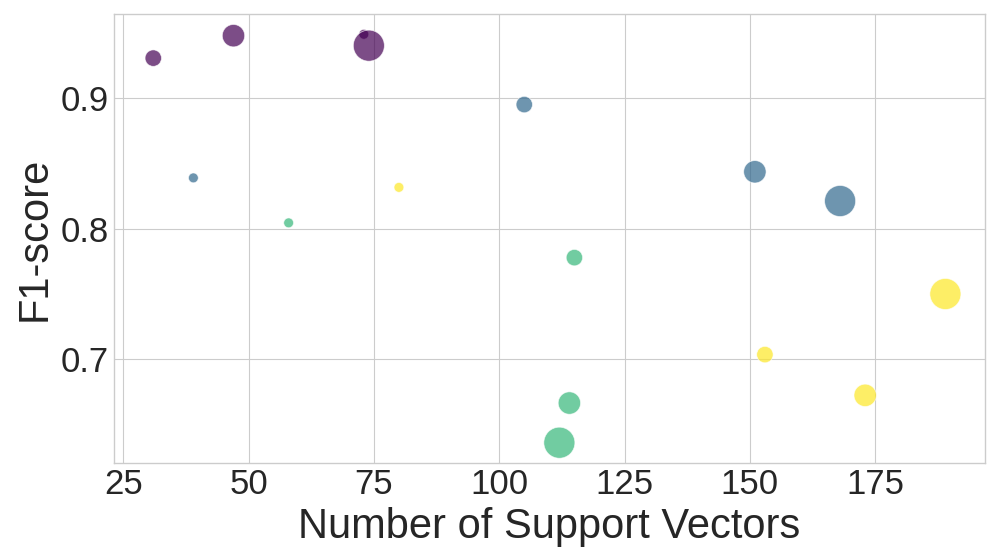}}
    \hfill
    \captionbox*{ \scriptsize{L.F (\acsincome{} and 3-MLP)}   }
    {\includegraphics[width=0.245\textwidth]{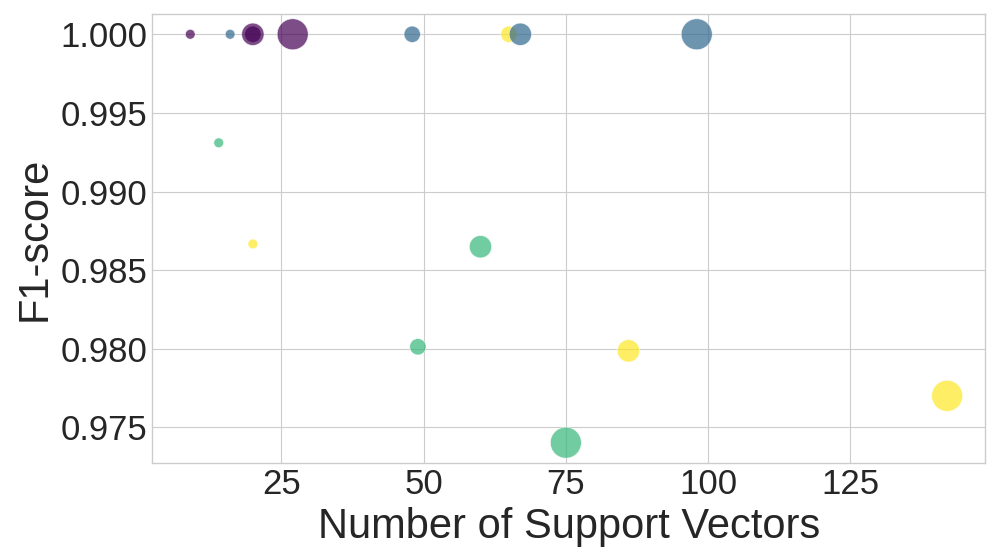}}
    \hfill
    \captionbox*{ \scriptsize{L.S (\femnist{} and LeNet5)}   }
    {\includegraphics[width=0.245\textwidth]{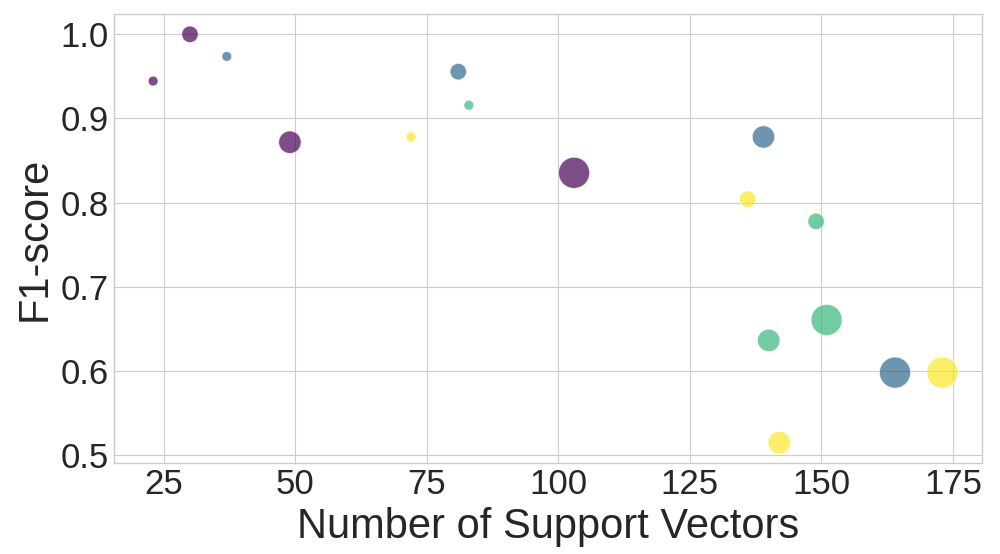}}
    \caption{Grid search over optimal layers for each architecture. Bubbles' color and size represent SVM's kernel degree, and SPCA projection dimension, respectively. Optimal configurations w.r.t $\mathcal{C}^{\mathsf{FHE}}_{\mathsf{SVM}}$ are the small purple bubbles in the upper-left corner of each plot.}
\label{fig:bubble_plot_grid_search}
\end{figure*}

\begin{table}[ht!]
\caption{Best filters per attack based on grid search using 
$\mathcal{C}^{\mathsf{FHE}}_{\mathsf{SVM}}$. 
Layer notation follows: I = input layer, Ck = conv layer k, Hk = hidden FC layer k. 
\textbf{L.F} and \textbf{L.S} refer to label flipping and label shuffling.}
\label{tab:best_settings}
\centering
\footnotesize 
\begin{tabular}{lcccc}
\toprule
\textbf{Parameter} & \textbf{Backdoor} & \textbf{L.F} & \textbf{Grad. Ascent} & \textbf{L.S} \\
\midrule
Layer & C1 & H3 & Last FC & C2 \\
$\mathsf{dim}$ & 8 & 64 & 8 & 8 \\
$n_{\mathsf{SV}}$ & 48 & 76 & 27 & 55 \\
Degree & 1 & 2 & 1 & 1 \\
\midrule
F1-score & 90.3\% & 91.5\% & 94.1\% & 89.2\% \\
\bottomrule
\end{tabular}
\end{table}


\section{Byzantine Filtering with CKKS}
\label{sec:ckks_building_blocks_for_filtering}
In this section, we detail the building blocks necessary to efficiently implement the FHE-filtering process, which includes SPCA projection and SVM inference.

\subsection{Matrix-Vector product and slot summation}
\label{sec:ckks_spca_projection}

Our filtering method uses batched matrix-vector multiplications under the CKKS scheme in two key phases. 
First, we project the encrypted layer weights $\llbracket w \rrbracket$ using SPCA by computing $ \llbracket v \rrbracket = A \cdot \llbracket w \rrbracket$, in which $A \in \mathbb{R}^{\mathsf{dim} \times \mathsf{d}}$ is the SPCA transformation matrix of $\mathsf{dim}$ rows and $\mathsf{d}$ columns ($\mathsf{d}$ is the number of layer weights). 
Then, we carry out the SVM hyperplane projection by evaluating component-wise $(M \cdot \llbracket v \rrbracket + \beta)^{\mathsf{deg}}$, in which $M \in \mathbb{R}^{n_{\mathsf{sv}} \times \mathsf{dim}}$. 

A key challenge in these matrix-vector operations is minimizing the number of rotations, a computationally expensive operation in CKKS.
Due to its importance in applications like privacy-preserving ML inference and bootstrapping, optimizing this operation is a central focus of the CKKS and related RLWE-based schemes \cite{ckks_boot_2020_bossuat, PPML_with_CKKS, cryptoeprint:2023/1649, ebel2025orionfullyhomomorphicencryption}.
In this work, we use the method introduced by Halevi and Shoup \cite{cryptoeprint:2014/106} that relies on multiplying the rotated ciphertext by \emph{the general diagonals} of the matrix and accumulating the results. 

After the SVM hyperplane projection, we complete inference by adding all the slots of the ciphertext encrypting $(M \cdot \llbracket v \rrbracket + \beta)^{\mathsf{deg}}$, by using the $\mathsf{TotalSum}$ algorithm introduced by Halevi and Shoup~\cite{cryptoeprint:2014/106}. 
The latter uses rotations to shift the vector elements by powers of two, allowing partial sums to be accumulated across slots in logarithmic steps. 
By iteratively adding these rotated ciphertexts, all values are homomorphically summed and the result is replicated in every ciphertext slot.

\subsection{Polynomial approximation of Heavyside}
The binary prediction of an SVM is expressed as the sign of the projected sample on the decision hyperplane. 
Our approach involves \emph{nullifying} Byzantine updates and excluding them from the aggregation. To do so, we replace the standard $\mathsf{Sign}$ function with a $\mathsf{Heavyside}$ activation, defined as:  
\[
\mathsf{Heavyside}(x) =
\begin{cases} 
1 & \text{if } x \geq 0, \\
0 & \text{if } x < 0.
\end{cases}
\]

Approximating step functions such as $\mathsf{Heavyside}$ with polynomials results in higher approximation errors~\cite{cryptoeprint:2020/834,cryptoeprint:2021/315, cryptoeprint:2021/1215}. 
To address this, we replace the $\mathsf{Heavyside}$ function with a smoother alternative: a $\mathsf{Sigmoid}$ function with a scaled input by a (plaintext) factor $k$. 
Figure~\ref{fig:Sigmoid(kx)} compares various $\mathsf{Sigmoid}$ functions with respect to different scaling factors $k$.
\begin{figure}[ht!]
    \centering
    \includegraphics[width=0.8\linewidth]{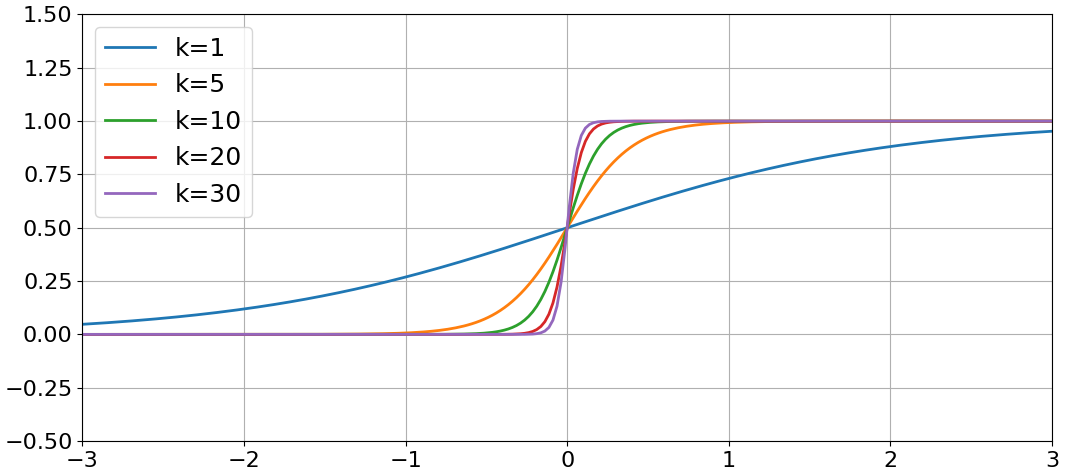}
    \caption{$\mathsf{Sigmoid}(kx)$ with $k \in \{1, 5, 10, 20, 30\}$.}
    \label{fig:Sigmoid(kx)}
\end{figure}
The $\mathsf{Sigmoid}$ function is approximated using a Chebyshev approximation on an interval around 0.  
\begin{figure}[ht!]
    \centering
    \includegraphics[width=0.8\linewidth]{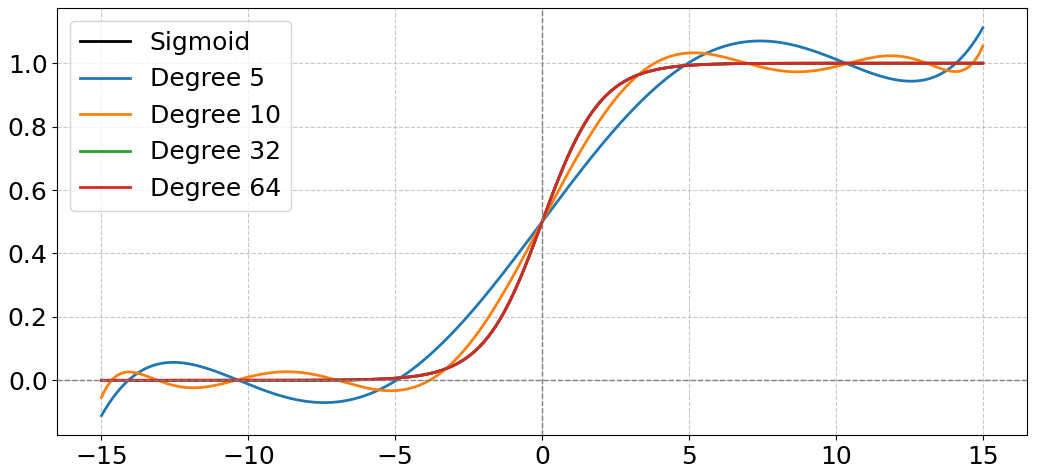}
    \caption{Chebyshev approximation with various degrees on $[-15, 15]$.}
    \label{fig:Sigmoid_approx(kx)}
\end{figure}

\paragraph{\textbf{Paterson-Stockmayer polynomial evaluation}}
To efficiently evaluate the Chebyshev approximation of $\mathsf{Sigmoid}$, we rely on Paterson-Stockmayer's baby-step giant-step approach. 
This algorithm optimizes the evaluation of a polynomial of degree $d$ by dividing the process into two stages:
\begin{enumerate}[leftmargin=*]
    \item \textbf{Baby steps.} Compute and store the powers of $x$: $\{x^1, x^2, \cdots, x^{m-1}\}$ with $m = \lceil \sqrt{d} \rceil $ 
    \item \textbf{Giant steps.} Compute and store the higher powers of $x^m$: $\{x^m, x^{2m}, \cdots, x^{m^2}\}$
    \item Every monomial of our polynomial is evaluated as a product of two elements from both the baby and giant steps sets. 
\end{enumerate}
This process reduces the multiplicative complexity of the polynomial evaluation from $O(d)$ to $O(\sqrt{d})$, which leads to two key benefits. 
First, since multiplications are the most computationally expensive operation in CKKS (and FHE in general), reducing the number of multiplications directly enhances efficiency. 
Second, by lowering the overall multiplicative depth of our FHE circuit, we can use more efficient CKKS parameters, reducing the cost of all unitary FHE operations.
For our Chebyshev approximation of $\mathsf{Sigmoid}$, we use $d=32$, which requires a homomorphic multiplicative depth of $8$.

\paragraph{\textbf{Smooth penalty instead of inclusion-exclusion}}
Using a smooth penalty as a filtering strategy, as opposed to a binary decision (through $\mathsf{Heavyside}$), offers several advantages, particularly in the presence of uncertainty in the filtering process. 
A binary approach, which strictly includes or excludes updates based on a hard threshold, can be overly rigid and fails to account for this uncertainty. 
In contrast, a $\mathsf{Sigmoid}$-based penalty provides a gradual transition between low and high weights, allowing for more nuanced adjustments. 
Also, the boundary is homomorphically adjustable to balance false positive and true negative rates. 
This is achieved via: 
$$
  \hat{y} = \mathsf{PolyHeavyside} \left(\left( \sum_{i=0}^{k-1} \alpha_i y_i \mathsf{Kern}( x_i,  \llbracket x \rrbracket) + \beta \right) + \mathsf{B}\right),  
$$
in which $\mathsf{B}\neq0$ denotes a new boundary value. 
Negative values induce less tolerance, in which higher confidence is needed to achieve $P^t_j \geq 0.5$, while positive values reduce the tolerance of the filtering process.
From an FHE perspective, this is a ciphertext-plaintext addition\footnote{The $\mathsf{B}$ value is encoded as a degree $N$ polynomial and summed directly to the second polynomial term of the RLWE ciphertext format (\emph{c.f.} Appendix \ref{sec:appendix_ckks}).}, which is computationally negligible and does not increase the ciphertext's noise, nor consumes FHE-computational levels.  
This flexibility allows the approach to be tailored to the specific property. 
For instance, when targeting gradient-ascent attackers, incorporating even a single Byzantine update into the aggregation can severely disrupt the achieved learning progress, making it a greater concern than temporarily down-weighting an honest worker for one round.

\subsection{Model Normalization with Newton-Raphson}

In Algorithm~\ref{alg:approach1}, each worker's update $\theta^t_i$ is scaled by $\frac{n_i}{\sum_{k\in [n]} n_k} \llbracket P^t_i \rrbracket$. 
This approach requires a normalization step after computing the sum of these scaled updates. 
Indeed, the scaling factors do not necessarily sum to $1$ (unless all the property filters are equal to $1$). 
The encrypted sum of the factors $\llbracket S^t \rrbracket = \sum_{i \in [n]} \left( \frac{n_i}{\sum_{k\in [n]} n_k} \right) \llbracket P^t_i \rrbracket$ is in $[0, 1]$. 
Hence, to readjust the aggregated model's parameters to the correct scale, a division by $S^t$ is required at each iteration. 
Since the division is not directly supported in the CKKS encryption, we approximate the inverse of $S^t$ using the Newton-Raphson method.

\paragraph{\textbf{Newton-Raphson for FHE division}}
The Newton-Raphson method aims to approach the root of a function $f$ via a converging sequence. 
It starts from an initial guess $x_0$ and updates it via the formula:
$$
x_{k+1} = x_k - \frac{f(x_k)}{f'(x_k)}.
$$
Given an encrypted scalar $\llbracket S^t \rrbracket$, approximating its inverse involves approximating the root of the function $f(x) = \frac{1}{x} - S^t$. 
The associated Newton-Raphson approximation sequence is: 
\begin{equation*}
    x_{k+1} = x_k (2 - \llbracket S^t \rrbracket x_k)
\end{equation*}
Selecting an appropriate initial guess \( x_0 \) is crucial for fast convergence. 
Since $S^t$ is in \([0,1]\) at each iteration \( t \), we employ a degree-1 Taylor approximation of $ \frac{1}{x} $ at the midpoint $(x = \nicefrac{1}{2})$. 
Hence, we initialize the Newton-Raphson sequence with $4 - 4 \llbracket S^t \rrbracket$, which requires only one plaintext-ciphertext multiplication and one subtraction. 
This approach minimizes the computational cost and avoids any multiplicative level consumption. 
Subsequently, each iteration requires two homomorphic multiplications, with a total number of FHE multiplications being $2k$.
Algorithm~\ref{alg:newton_raphson_fhe} outlines this method, while Figure~\ref{fig:newton_raphson_error_distrib} illustrates the distribution of approximation errors for the Newton-Raphson inverse computation (executed in cleartext) as a function of the number of iterations.

\begin{algorithm}[ht!]
\caption{ $\mathsf{FHEInverse}$ : Newton-Raphson FHE Inversion}
\label{alg:newton_raphson_fhe}
\begin{algorithmic}[1]
\Require Ciphertext $ \mathsf{ct}_{S^t} $ encrypting scalar $S^t \in [0, 1]$, evaluation key $\evk$, number of iterations $k$
\Ensure Ciphertext $\mathsf{ct}_{S^{-1}}$ encrypting an approximation of $(S^t)^{-1}$
\State \textbf{Initialize:} $\mathsf{ct}_x \gets \mathsf{PolyEval}(4  - 4x, \mathsf{ct}_{S^t})$  
\For{$ j \gets 1$ to $ k $}
    \State $\mathsf{ct}_{\mathsf{tmp}} \gets \mathsf{Mult}_{\evk}(\mathsf{ct}_S, \mathsf{ct}_x)$ \Comment{Compute $S^t x_k$}
    \State $\mathsf{ct}_{\mathsf{tmp}} \gets \mathsf{Sub}_{\evk}(2, \mathsf{ct}_{\mathsf{tmp}})$ \Comment{Compute $2 - S^t x_k$}
    \State $\mathsf{ct}_x \gets \mathsf{Mult}_{\evk}(\mathsf{ct}_x, \mathsf{ct}_{\mathsf{tmp}})$ \Comment{Update $x_{k+1}$}
\EndFor
\State \Return $\mathsf{ct}_x$
\end{algorithmic}
\end{algorithm}
Figure~\ref{fig:newton_raphson_error_distrib} illustrates that $k=2$ yields a sufficiently accurate inverse approximation, with an expected approximation error of around $10^{-2}$ and a computational cost of $4$ FHE products. 
Thus, we consistently use $k=2$ for $\mathsf{EncryptedInverse}(\cdot)$ in our aggregation circuit. 
\begin{figure}[ht!]
    \centering
    \includegraphics[width=0.9\linewidth]{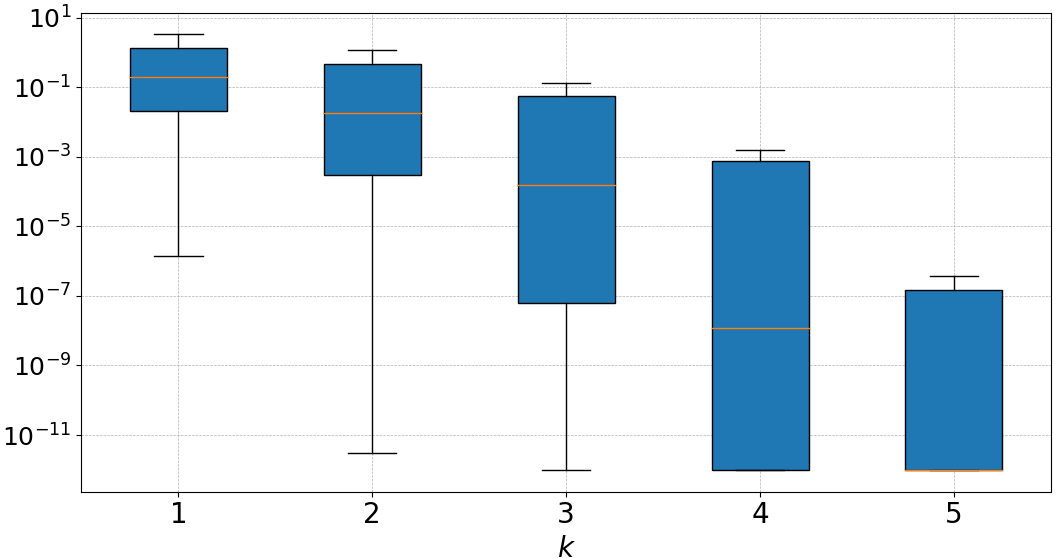}
    \caption{Distributions of $|\nicefrac{1}{x} - \mathsf{EncryptedInverse}(x)|$ over 100 seeds for $k \in \{1, \dots, 5\}$.}
    \label{fig:newton_raphson_error_distrib}
\end{figure}

Finally, Figure~\ref{fig:filtering_process} provides an illustration of the filtering process.
\begin{figure}[ht!]
    \centering
    \includegraphics[width=1.05\linewidth]{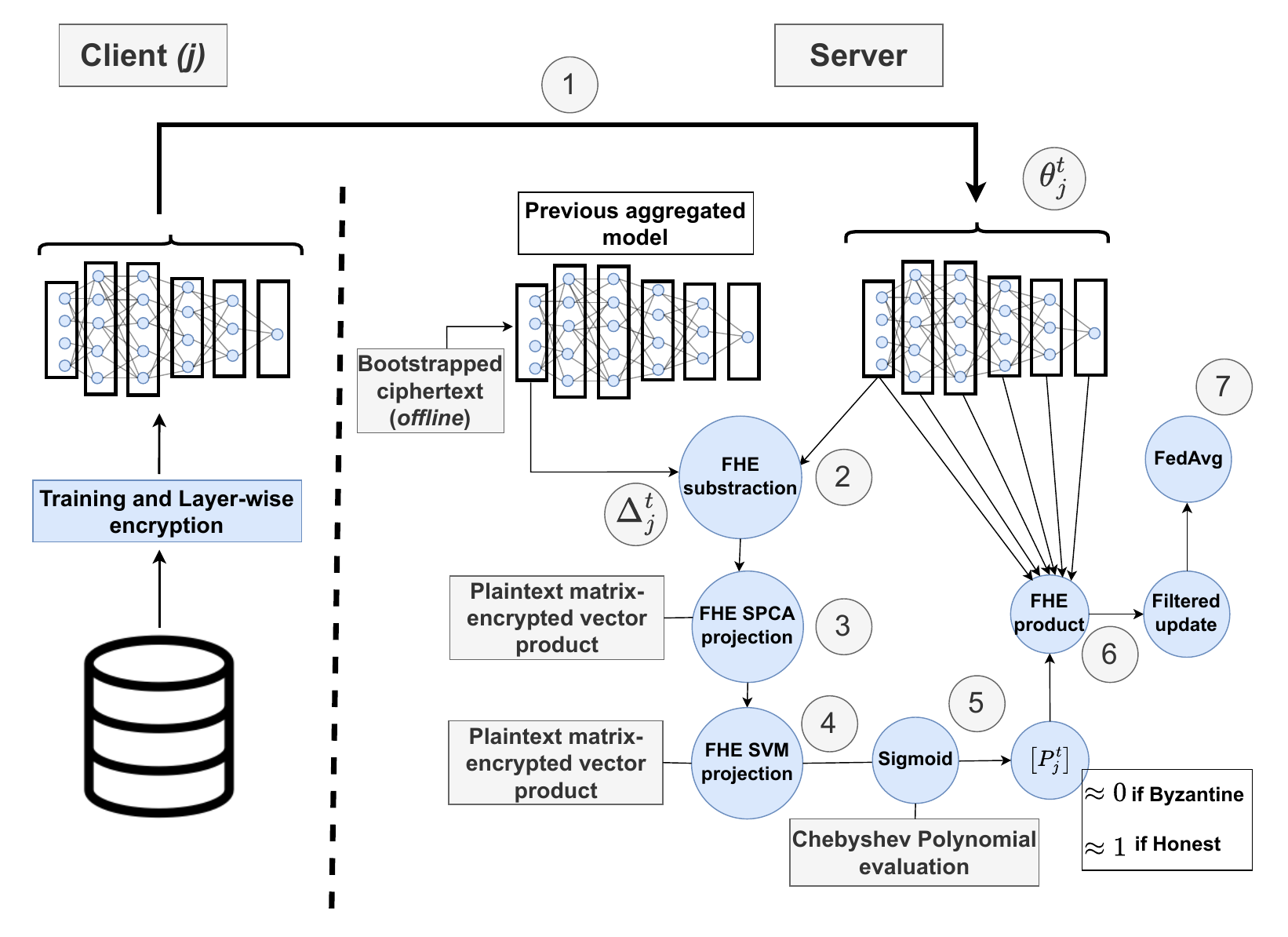}
    \caption{Filtering process in six steps. Steps 2 and 3 involve the plaintext-matrix encrypted-vector products. Step 5 is the Paterson-Stockmayer evaluation of the Chebyshev approximation.}
    \label{fig:filtering_process}
\end{figure}

\section{Experimental Results}
\label{sec:experimental}

We evaluate our approach experimentally along two primary aspects: (1) the computational overhead introduced by the FHE evaluation of the filtering process, as detailed in Section \ref{sec:fhe_inference_exp}, and (2) the impact of the update-filtering mechanism on the FL convergence under highly adversarial conditions, discussed in Section \ref{sec:fl_with_svm_filters}. 
The FHE performance experiments were conducted on a 12th Gen Intel(R) Core(TM) i7-12700H CPU and 22 GiB of RAM, using the $\mathsf{Tenseal}$ and $\mathsf{Lattigo}$ homomorphic libraries \cite{benaissa2021tenseallibraryencryptedtensor, lattigo}.

\subsection{Encrypted property-filtering}
\label{sec:fhe_inference_exp}
We use two sets of CKKS parameters for performance analysis: $\mathsf{P1}$ for linear SVMs and $\mathsf{P2}$, which provides an additional multiplicative level ($L$) required for degree-2 polynomial SVMs. 
\begin{table}[ht!]
\scriptsize
\caption{Two sets of CKKS parameters and their associated security levels ($\lambda$) estimated via LWE-estimator \cite{cryptoeprint:2015/046}.}
\centering
\begin{tabular}{c|c|c}
\hline
\textbf{Set} & \textbf{CKKS Parameters} & \textbf{Security  $(\lambda)$} \\ \hline
$\mathsf{P1}$ & \(\log_2(N) = 14, \, \log_2(Q) = 340, \, L=11, \, \sigma = 3.19\) & $\approx$ 135 bits \\ \hline
$\mathsf{P2}$ & \(\log_2(N) = 15, \, \log_2(Q) = 500, \, L=12\ , \sigma = 3.19\) & $\approx$ 206 bits \\ \hline
\end{tabular}
\label{tab:ckks_parameters}
\end{table}

\begin{table}[ht!]
    \footnotesize
    \centering
    \caption{Runtime (s) of FHE building blocks with associated CKKS Parameters.}
    \begin{tabular}{l|cccc}
        \toprule
        \textbf{Operation} & \textbf{Backdoor} & \textbf{Grad. Ascent} & \textbf{L.F} & \textbf{L.S} \\
        \textbf{Parameters} & $\mathsf{P1}$ & $\mathsf{P1}$ & $\mathsf{P2}$ & $\mathsf{P1}$ \\
        \midrule
        SPCA Proj. & 4.10 & 4.18 & 14.84 & 3.98 \\
        SVM Proj. & 1.45 & 0.93 & 3.07 & 1.78\\
        $\mathsf{Sigmoid}$ & 1.21 & 1.23 & 2.85 & 1.35 \\
        Enc. Inv. & 1.37 & 1.41 & 3.62 & 1.27 \\
        $\mathsf{Bootstrapping}$ & 22.87 & 24.12 & 32.67 & 25.03 \\
        \bottomrule
        Total $\mathsf{Filter}_P$ & 6.77 & 6.35 & 24.38 & 7.11 \\
        \bottomrule
        Total Agg. & 8.34 & 8.36 & 26.07 & 8.95 \\
        \bottomrule
    \end{tabular}
    \label{tab:fhe_svm_runtimes}
\end{table}

\paragraph{\textbf{Findings}}
Table~\ref{tab:fhe_svm_runtimes} illustrates the practical applicability of our filtering method, showing that most FHE operations (excluding bootstrapping) are completed in $0.9$ to $\approx 15$ seconds. 
The SPCA projection, which processes raw, high-dimensional encrypted update vectors, is the most computationally expensive step. 
In contrast, the SVM hyperplane projection is much faster thanks to its lower input dimension. 
The Paterson-Stockmeyer optimization significantly speeds up the evaluation of the $32$-degree Chebyshev approximation of the $\mathsf{Sigmoid}$ function. 
As a result, the overall filter computation time remains under $9$ seconds for Backdoor, gradient-ascent, and label-shuffling SVMs. 
Label-flipping filtering requires a quadratic kernel SVM with a larger number of support vectors and parameters ($\mathsf{P2}$), leading to an increased computation time. 
Although bootstrapping is costly, it is performed on the server side and can overlap with worker-side operations, thereby reducing its impact on the FL workflow. Further details are given in Appendix~\ref{sec:appendix_bootstrapping_offline_impact}.

\subsection{Robust FL via SVM-filtering}
\label{sec:fl_with_svm_filters}
\subsubsection*{\textbf{Shadow updates generation}}
We use the approach outlined in Algorithm~\ref{alg:filtering_data_generation} to generate the shadow updates used to train the SPCA+SVM filter. 
The data used to train shadow updates $D_{\text{shadow}}$ is entirely distinct from the data used in the FL experiments outlined below. Specifically, for \femnist{}, we use the conventional \textsf{MNIST} dataset. 
For \cifar{} and \GTSRB{}, we uniformly draw 10K samples from each dataset. For \acsincome{}, we use data from the state of California. 
We run 25 complete trainings of 50 epochs each, with a Byzantine fork at each epoch. For the backdoor shadow updates, each run uses a new backdoor configuration. 
We insert a random-size trigger (between 4 and 8 pixels) at a random image location and assign a randomly chosen target label. 
This variability simulates the server’s limited prior knowledge of the attackers’ backdoor strategy. 
Each training uses a different split of $D_{\text{shadow}}$ obtained by uniform Dirichlet sampling ($\alpha = 1$) over the label distribution, and a different initial model ($\theta^{\text{init}}$) via Xavier Glorot initialization \cite{pmlr-v9-glorot10a}. 
\subsubsection*{\textbf{Experimental setup}}
For each Byzantine behavior, we set 10 workers with a 50\% attacker proportion. 
We run 50 \fedavg{} rounds on the \femnist{}, \GTSRB{} \cifar{}, and \acsincome{} benchmarks. 
We compare performance with and without our SVM filtering.
The \femnist{} and \acsincome{} datasets inherently provide favorable FL setups: \footnote{\femnist{} (Federated MNIST) via its different handwriting styles and \acsincome{} through its 50 American states demographic distributions.}. 
For the \cifar{} and \GTSRB{} benchmarks, we use a Dirichlet sampling approach \cite{hsu2019measuring, bendoukha2025} with $\alpha = 0.5$ for a moderate label-heterogeneity level. 
A baseline without attackers is also included. 
Each experimental configuration is repeated $5$ times, with max, min and average values being reported at each round for a robust evaluation.

\begin{enumerate}[leftmargin=*]
    \item \textbf{Backdoor attacks.} Byzantine workers update the global model on datasets with 90\% backdoor-embedded samples (Figure \ref{fig:backdoor_illustration}). 
    We measure (1) the genuine \GTSRB{} test accuracy and (2) backdoor test accuracy, which assesses the attackers' success in inducing the misclassification of manipulated data. 
    \item \textbf{Gradient ascent, label flipping and shuffling.} Gradient-ascent attackers invert gradients' direction, while label flipping and shuffling attackers manipulate class labels. 
    In these attacks, we assess global model test accuracy to measure the attacker's success and SVM defense.
\end{enumerate} 

Finally, we investigate the dynamic behavior of the SVM filter throughout an FL run by comparing the filter score distributions for honest and malicious updates.
\begin{figure}[ht!]
    \centering
    \includegraphics[scale=0.37]{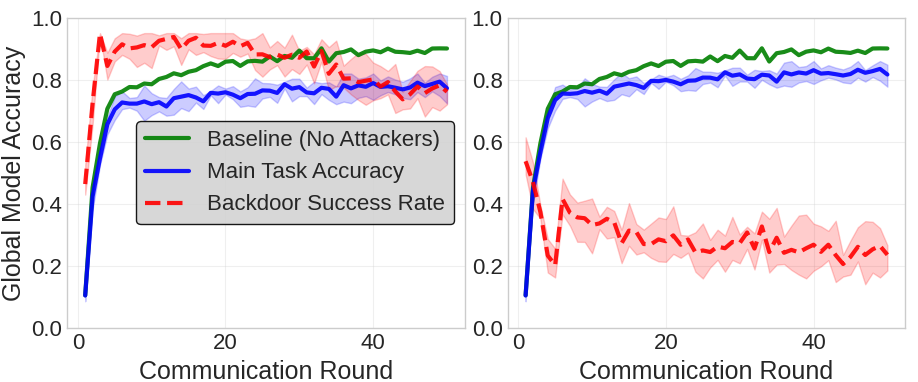}
    \caption{Backdoor: With (right) and without (left) SVM filter, genuine and backdoor test accuracy.}
    \label{fig:backdoor_filtering}
\end{figure}

\begin{figure}[H]
    \centering
    \includegraphics[width=0.9\linewidth]{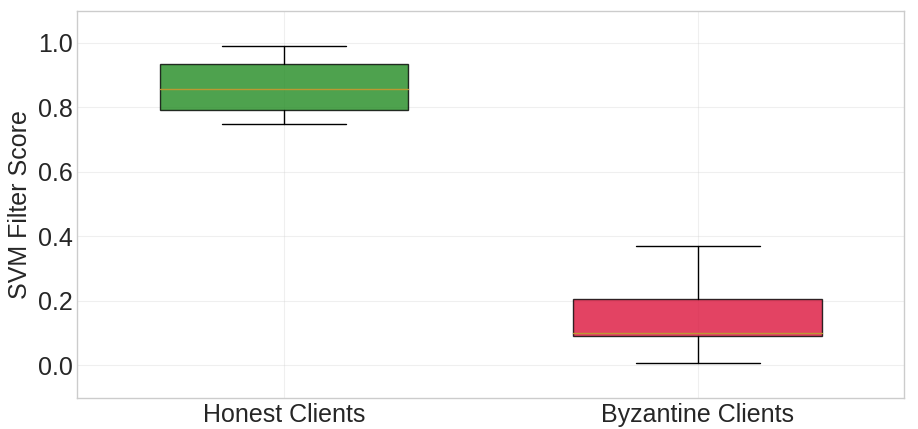}
    \caption{Filter performance across all 50 rounds, and 5 seeds corresponding to the experiments in Figure~\ref{fig:backdoor_filtering}} 
    \label{fig:backdoor_filtering_perf}
\end{figure}

\subsubsection*{\textbf{Findings}} Below, we summarize our main findings.
\begin{itemize}[leftmargin=*]
    \item{\textbf{Backdoor}.} Figure \ref{fig:backdoor_filtering} (left) shows that, without defense, a 50\% presence of backdoor-embedded updates rapidly compromises the main learning task, achieving near-perfect test accuracy for the backdoor objective (classifying signs as \emph{``End of Speed Limit''} in presence of the backdoor pattern), and reducing the accuracy on the main task by about 15\% due to the contradictory nature of the two tasks. 
    In contrast, Figure~\ref{fig:backdoor_filtering} (right) demonstrates the effectiveness of our SVM-filtering approach, limiting the backdoor task's accuracy to below 25\%. 
    This is achieved via the systematic down-weighting of updates from Byzantine workers, effectively minimizing their impact on the global model at each iteration.
    \item{\textbf{Gradient-ascent, label-flipping and label-shuffling}.} 
    When 50\% of workers perform gradient-ascent and label flipping or label-shuffling, convergence is severely compromised as observed in Figure~\ref{fig:three_attacks_conv}. 
    Test accuracy plateaus respectively at 10\% and 45\%, indicating that Byzantine workers effectively undermine the learning progress. 
    However, the high severity of these attacks leads to accurate detection by the SVM, resulting in a significant down-weighting of these updates, with $P_i^t$ values approaching zero (\cf{}, Figure \ref{fig:three_attacks_filter_perf}). 
    Consequently, eliminating these harmful updates restores convergence to near-baseline levels (the no-attacker scenario).  

\end{itemize}

\begin{figure*}[ht!]
    \centering
    \captionbox*{\small{(a) \acsincome{} 3-MLP}}{\includegraphics[width=0.325\textwidth]{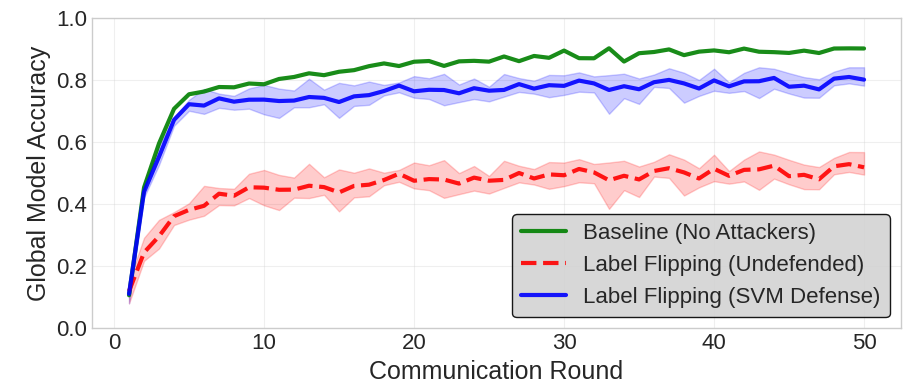}}
    \hfill 
    \captionbox*{\small{(b) \cifar{}{} ResNet-14}}{\includegraphics[width=0.325\textwidth]{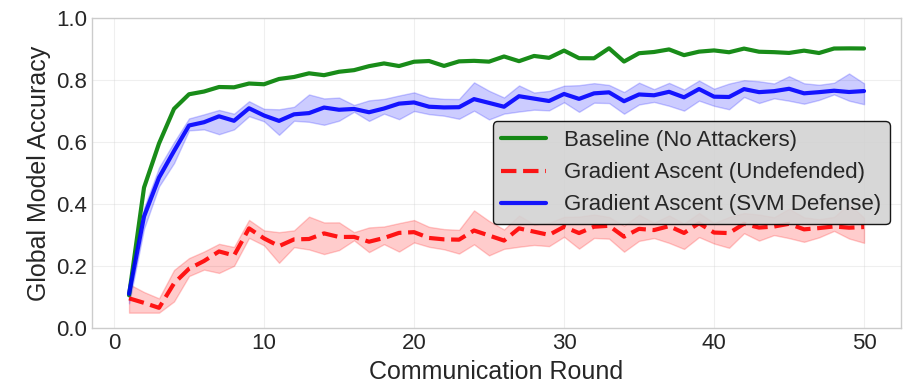}}
    \hfill 
    \captionbox*{\small{(c) \femnist{} LeNet-5}}{\includegraphics[width=0.30\textwidth]{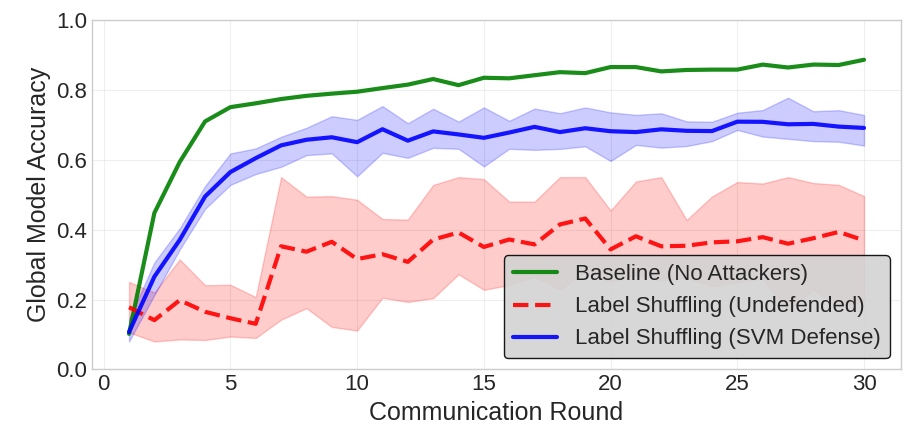}}

    \caption{Convergence of our Byzantine filtering approach over the \acsincome{}, \cifar{}, and \femnist{} benchmarks with a 3-MLP, ResNet-14, and LeNet-5 models.}
    \label{fig:three_attacks_conv}
\end{figure*}

\begin{figure*}[ht!]
    \centering
    \captionbox*{\small{(a) \acsincome{} 3-MLP}}{\includegraphics[width=0.325\textwidth]{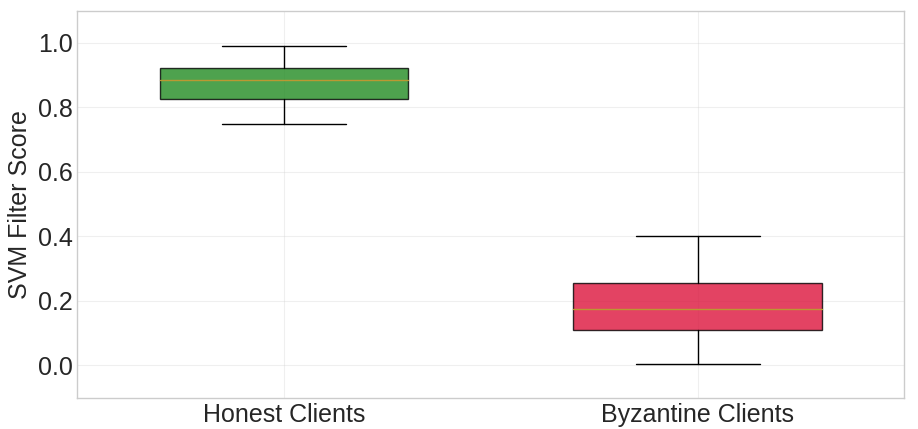}}
    \hfill 
    \captionbox*{\small{(b) \cifar{}{} ResNet-14}}{\includegraphics[width=0.325\textwidth]{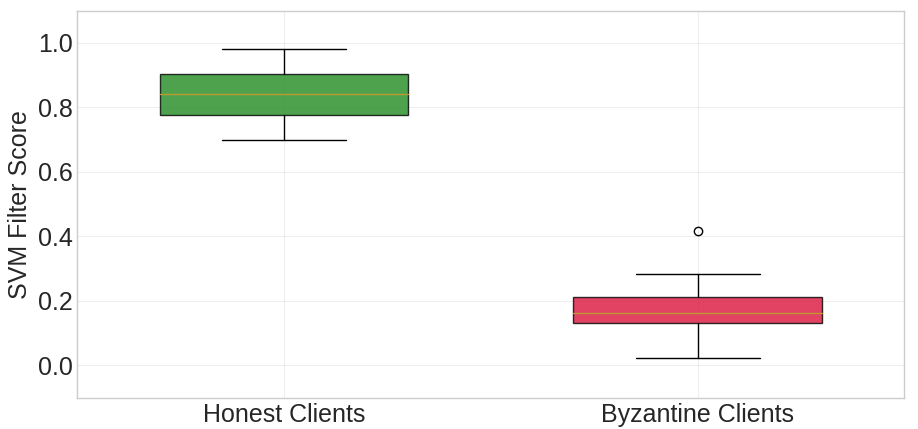}}
    \hfill 
    \captionbox*{\small{(c) \femnist{} LeNet-5}}{\includegraphics[width=0.30\textwidth]{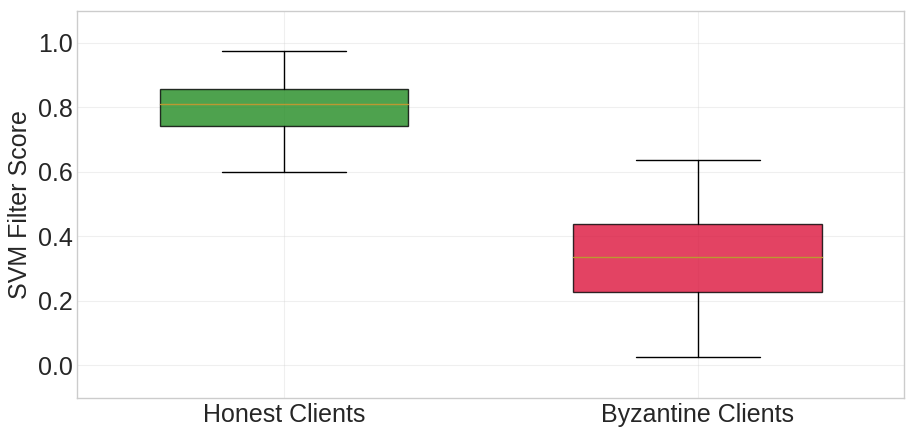}}

    \caption{Filter performance across all 50 rounds, and 5 seeds corresponding to the experiments in Figure \ref{fig:three_attacks_conv}}
    \label{fig:three_attacks_filter_perf}
\end{figure*}

\textbf{Filter values dynamic.}
Figures \ref{fig:backdoor_filtering_perf} and \ref{fig:three_attacks_filter_perf} display the distribution of filter values during the experiments of Figures \ref{fig:backdoor_filtering} and \ref{fig:three_attacks_conv}.
The key observation is that, consistent with our theoretical analysis, the filter quality, captured by $(\mu_h, \mu_b, \sigma_h, \sigma_b)$, is the dominant factor governing convergence toward the no-attacker baseline.
For example, the label-shuffling filter exhibits the weakest separability, with high $\mu_b$ and $\sigma_b$.
Consequently, it yields the largest accuracy gap ($\approx 10\%$) relative to the no-attacker baseline, which matches the tolerance bounds $\epsilon$ and $\epsilon_{\text{base}}$ stated in Corollary \ref{cor:tolerable}. 
In contrast, the backdoor filter achieves substantially better down-weighting of Byzantine updates (low $\mu_b$ and $\sigma_b$), resulting in a much smaller accuracy gap of roughly $5\%$.

\section{Conclusion}
\label{sec:conclusion}

This work presents a novel approach that adapts methods for property inference attacks to detect workers' misbehaviors in privacy-preserving FL with fully homomorphic encryption. 
By relying on FHE-friendly filtering models and capitalizing on the flexibility of the CKKS encryption scheme, our method achieves practical compatibility with CKKS-based secure aggregation. 
Empirical results demonstrate the efficiency of our approach in detecting diverse Byzantine activities, which inevitably generate anomalies reflected in the successive model updates that can be inferred, thereby revealing their malicious nature. 
This perspective shifts the conventional understanding of privacy leakage in FL from a vulnerability to a valuable signal for resilience. 
As future works, we will investigate the extension of this methodology to other sets of homomorphic encryption schemes and explore the detection of increasingly complex Byzantine attack vectors via FHE-friendly machine-learning models and techniques. 
Our approach will also benefit from the ongoing improvement of the CKKS bootstrapping and the hardware acceleration of FHE schemes.

\bibliographystyle{unsrt}
\bibliography{biblio.bib}

\appendices

\section{Cryptographic background}
\label{sec:appendix_ckks}

\subsection{The CKKS scheme}

The CKKS scheme \cite{CKKS_2017} is a \emph{leveled} homomorphic encryption technique designed for approximate arithmetic on vectors of complex or real numbers, natively supporting element-wise additions, multiplications and slot rotations. 
As a leveled scheme, CKKS ciphertexts are limited to a fixed number of multiplications $L$. 
Exceeding this depth prevents correct decryption. However, this limitation is overcome by bootstrapping. 
Recent works \cite{ckks_boot_2020_bossuat, ckks_boot_2022_bossuat, ckks_boot_2024_bae, ckks_boot_2025_min} have significantly advanced the efficiency of CKKS bootstrapping procedures, effectively removing the constraint on multiplicative depth. 
The main operations of the CKKS scheme are described below.
\begin{itemize}[leftmargin=*]
    \item{$\mathsf{ParamsGen}(\lambda, L, \text{precision})$:} Generates the public parameters for the CKKS scheme based on the security parameter $\lambda$, the desired multiplicative depth $L$, and the target precision for approximate arithmetic. Outputs typically include: the polynomial ring degree $N$ (a power of 2, determining the vector length), the chain of ciphertext coefficient moduli $Q = q_0 \cdot q_1 \cdots q_L$ (where $Q_l = q_0 \cdots q_l$ is the modulus at level $l$), an auxiliary modulus $P$ used for key switching (relinearization, rotation), the scaling factor $\Delta$ (often a power of 2, related to precision), and specifications for the secret key distribution $S$ and error distribution $\chi$.
    These parameters must ensure $\lambda$-bit security while supporting $L$ levels of multiplication and the desired precision.
    \item{$\mathsf{KeyGen}(\text{params})$:} Samples the secret key polynomial $s$ from the distribution $S$. 
    The secret key is $\sk = (1, s)$. 
    The public key $\pk = (b, a)$ is generated by sampling $a$ uniformly from $R_Q = \mathbb{Z}_Q[X]/(X^N+1)$, sampling an error $e$ from $\chi$, and setting $b = [-as + e]_Q$.
    \item{$\mathsf{RelinKeyGen}(\sk)$:} Generates the relinearization key $\mathsf{rlk}$, which is necessary to reduce the size of ciphertexts after multiplication. 
    It is effectively an encryption of $s^2$ under the secret key $\sk$, typically constructed using the extended modulus $PQ$. 
    A common form involves sampling $a' \leftarrow R_{PQ}$, $e' \leftarrow \chi$, yielding $\mathsf{rlk} = (b', a') = ([-a's + e' + Ps^2]_{PQ}, a')$.
    \item{$\mathsf{RotKeyGen}(\sk, k)$:} Generates a rotation key $\mathsf{rotk}_k$ required for homomorphically rotating the encrypted vector by $k$ positions (where $k$ corresponds to a specific Galois element). 
    This key effectively encrypts the modified secret key $s(X^k)$ under the original secret key $\sk$, typically using the modulus $PQ$. 
    Keys must be generated for all desired rotation indices $k$.

    \item{$\mathsf{Encryption}(\pk, m)$:} To encrypt a vector of complex/real numbers $m$, it's first encoded into a plaintext polynomial $\tilde{m} \in R$. 
    Then, sample $v \leftarrow S$ and errors $e_0, e_1 \leftarrow \chi$. Given $\pk=(b, a)$, the ciphertext $c=(c_0, c_1)$ is computed as $([ \Delta \cdot \tilde{m} + b v + e_0 ]_{Q_L}, [a v + e_1]_{Q_L})$, in which $Q_L$ is the initial ciphertext modulus and $\Delta$ is the scaling factor.

    \item{$\mathsf{Decryption}(c, \sk)$:} Given a ciphertext $c = (c_0, c_1)$ at level $l$ (modulus $Q_l$) and $\sk = (1, s)$, compute the scaled plaintext polynomial $m' = [c_0 + c_1 s]_{Q_l}$. 
    The approximate plaintext vector is obtained by decoding $m' / \Delta$.

    \item{$\mathsf{Addition}(c_a, c_b)$:} Given two ciphertexts $c_a = (c_{a,0}, c_{a,1})$ and $c_b = (c_{b,0}, c_{b,1})$ at the same level $l$, their homomorphic sum is computed component-wise: $c_{add} = ([c_{a,0} + c_{b,0}]_{Q_l}, [c_{a,1} + c_{b,1}]_{Q_l})$.

    \item{$\mathsf{Multiplication}(c_a, c_b)$:} Homomorphic multiplication involves several steps:
        \begin{enumerate}[leftmargin=*]
            \item $\mathsf{Tensor Product}$: Calculate the polynomial products $d_0 = [c_{a,0} c_{b,0}]_{Q_l}$, $d_1 = [c_{a,0} c_{b,1} + c_{a,1} c_{b,0}]_{Q_l}$, and $d_2 = [c_{a,1} c_{b,1}]_{Q_l}$. The resulting ciphertext $(d_0, d_1, d_2)$ encrypts the product but depends on $s^2$.
            \item $\mathsf{Relinearization}$: Apply the relinearization key $\mathsf{rlk}$ to the $d_2$ term using a key-switching procedure, effectively replacing the $s^2$ dependency. 
            Add the result to $(d_0, d_1)$ to obtain a standard-form 2-component ciphertext $c_{relin}$ encrypting the product under $\sk$.
            \item $\mathsf{Rescaling}$: Divide the coefficients of $c_{relin}$ by the scaling factor $\Delta$ (usually via modular arithmetic and rounding), and reduce the ciphertext modulus from $Q_l$ to $Q_{l-1}$ (dropping $q_l$).
            This controls the scale of the plaintext and manages noise growth. Let the result be $c_{mult}$.
        \end{enumerate}
        The complete multiplication operation outputs $c_{mult}$.

    \item{$\mathsf{Rotate}(c, k)$:} Performs a homomorphic rotation (cyclic shift) of the slots in the encrypted plaintext vector by $k$ positions. 
    This utilizes the Galois automorphism $\sigma_k: p(X) \mapsto p(X^k)$. 
    The operation takes a ciphertext $c=(c_0, c_1)$ and applies $\sigma_k$ to its components. 
    Then, using the corresponding rotation key $\mathsf{rotk}_k$, it performs a key-switching operation to transform the result $(c_0(X^k), c_1(X^k))$ (decipherable with $s(X^k)$) back into a ciphertext decipherable with the original secret key $s(X)$. 
    This enables powerful SIMD (Single Instruction, Multiple Data) operations on encrypted data.

    \item{$\mathsf{Boot}(c)$:} Refreshes a ciphertext $c=(c_0, c_1)$ at a low level $l$, resetting its noise and restoring it to the maximum level $L$. This enables computations beyond the initial depth limit. 
    It homomorphically approximates decryption and re-encryption. We adopt the efficient method for non-sparse keys from Bossuat \etal. \cite{ckks_boot_2022_bossuat}, which has the high-level following steps:
    \begin{enumerate}[leftmargin=*]
        \item $\mathsf{ModUp}$: Increase the ciphertext modulus from $Q_l$ to a larger modulus $Q_{L+k}$ (where $k$ depends on bootstrapping depth) suitable for the homomorphic operations, typically by scaling the ciphertext components.
        \item $\mathsf{EvalDec}$: Approximate the decryption $m' = [c_0 + c_1 s]_{q_0}$ homomorphically, where $q_0$ is the first prime in $Q$ (related to initial scaling $\Delta$). This is the core part and usually includes:
            \begin{itemize}[leftmargin=*]
                \item $\mathsf{CoeffToSlot}$: Apply homomorphic linear transforms (analogous to inverse NTT/DFT) to make plaintext slots accessible.
                \item $\mathsf{EvalMod}$: Homomorphically evaluate a function that approximates the modular reduction $[c_0 + c_1 s]_{q_0}$. This often uses polynomial approximations (\emph{e.g.}, of a sine wave or another periodic function) to remove the scaled noise term involving $s$. 
                This step requires a bootstrapping key $\mathsf{bsk}$, which encrypts $s$ (or functions of $s$) under the bootstrapping modulus $Q_{L+k}$.
                \item $\mathsf{SlotToCoeff}$: Apply forward linear transforms (analogous to NTT/DFT) to return to coefficient representation.
            \end{itemize}
        \item $\mathsf{ModDown:}$ Process the result (which approximates the original scaled message $\Delta \tilde{m}$ modulo $q_0$) as a new plaintext, scale it appropriately and adjust the modulus down to $Q_L$, producing the final bootstrapped ciphertext $c_{boot}$ with low noise.
    \end{enumerate}

 \end{itemize}

\subsection{Threshold multi-key RLWE-based schemes}
\label{sec:appendix_thCKKS}

Threshold cryptography refers to encryption schemes with collaborative decryption mechanisms. That is, each participant holds a share $\sk_i$ of the (global) secret key $\sk$, which grants him the ability to perform a partial decryption of a ciphertext $\hat{m}_i =  \mathsf{Partial\_Dec}(\sk_i, c)$. An access structure $S \subset \mathsf{PowerSet}(P)$ is the collection of all subsets of parties with the ability to recover the plaintext message using their secret key shares. A $t$-out-of-$n$ scheme refers to the access structure containing subsets of size $t$. 

$\forall P \subset S$ we have $m = \dec(\sk, c) = \mathsf{Full\_Dec}(P_{Dec})$ 

with $P_{Dec} = \{\hat{m}_i \text{\,\,}\forall i \text{\, with \,} p_i \in P\}$, the set of contributed partial decryptions from parties in $P$.

The fundamental security assumption of threshold schemes is that any strict subset of $P \in S$ provides no additional knowledge regarding $m$ beyond what is already provided by the ciphertext $c$.

Mouchet \etal \cite{cryptoeprint:2022/780} introduced an additive \footnote{$\sk = \sum_{i \in [n]}\sk_i$, refers to encryption schemes with collaborative decryption a mechanism. The decryption consists of a product of the secret with the ciphertext is compatible with addition $<c, \sum_{i \in [n]}\sk_i> =  \sum_{i \in [n]} <\sk_i, c>$} $n$-out-of-$n$ multi-party construction for ring learning with errors structure, and therefore compatible with all RLWE-based homomorphic schemes. Namely, BGV \cite{BGV2011}, BFV \cite{BFV2012} and CKKS \cite{CKKS_2017}. It was later extended to $t$-out-of-$n$ access structures for any $t \leq n$ using the \textit{share-re-sharing} technique by Asharov \etal \cite{cryptoeprint:2011/613} that enables to move from an $n$-out-of-$n$ construction to a $t$-out-of-$t$ one, by \emph{re-sharing} the additive locally generated secret key through Shamir secret sharing, and taking advantage of the commutative nature between additive decryption reconstruction, and the interpolation of the secret key. That is, 

\begin{equation}
    \label{eq:threshold_key_reconstruction}
    s = \sum_{i=1}^{n} s_i = \sum_{i=1}^{n} \sum_{j=1}^{t} S_i(\alpha_j)\lambda_j = \sum_{j=1}^{t} \lambda_j\sum_{i=1}^{n} S_i(\alpha_j) = \sum_{j=1}^{t} s'_j
\end{equation}

Where $S_i$ denotes the Shamir polynomials of each participant, and $\alpha_i$ their Shamir public points. $s'$ are the $t$-out-of-$t$ reconstructed additive keys by the $t$ participants. 

Therefore, the initialization of this threshold construction involves, for each worker $j$ : 
\begin{enumerate}[leftmargin=*]
    \item Generate an additive secret key $s_j$ via uniform sampling in $\mathbb{R}_Q[X]/(X^N + 1)$
    \item Sample a degree $t-1$ bi-variate polynomial with coefficients ($c_1, \cdots, c_{t-1})$ in $\mathbb{R}_Q[X]/(X^N + 1)$ which will enable sharing the workers' additive key $s_j$. 
    \item Transmit to all other workers $(i)$ the share $S_j(\alpha_i) = s_j + \sum_{k=1}^{t-1} c_i \cdot\alpha^k_i$. That is, the evaluation of its sampled polynomial on their respective Shamir public points $\{\alpha_i\}_{i\neq j}$.
    \item After receiving the shares from the $n-1$ other workers $\{S_i(\alpha_j)\}_{(i \neq j)}$ (Their private Shamir polynomial evaluated on worker $j$'s public point). Worker $j$ computes the sum $\sum_{i=1}^{n} = S_i(\alpha_j)$.
\end{enumerate}
Finally, When at least $t$ workers are active, their $t$ Lagrange coefficients $\{\lambda_i\}_{i\in [t]}$ can be computed. Therefore, reaching the $t$-out-of-$t$ additive shares.

\section{Offline amortized cost of Bootstrapping}
\label{sec:appendix_bootstrapping_offline_impact}
We aim to evaluate the true impact of the bootstrapping operation to refresh the aggregated model's ciphertext $s$ in our FL workflow (Algorithms \ref{alg:approach1}). Specifically, we aim to compare the runtime of the bootstrapping operation with worker-side actions —mainly the collaborative decryption and the local updating of the model. Figure \ref{fig:temporal_diagram_bootstrapping} illustrates the expected amortized cost of the bootstrapping within the FL workflow.

\begin{figure}[ht]
    \centering
    \includegraphics[width=1.0\linewidth]{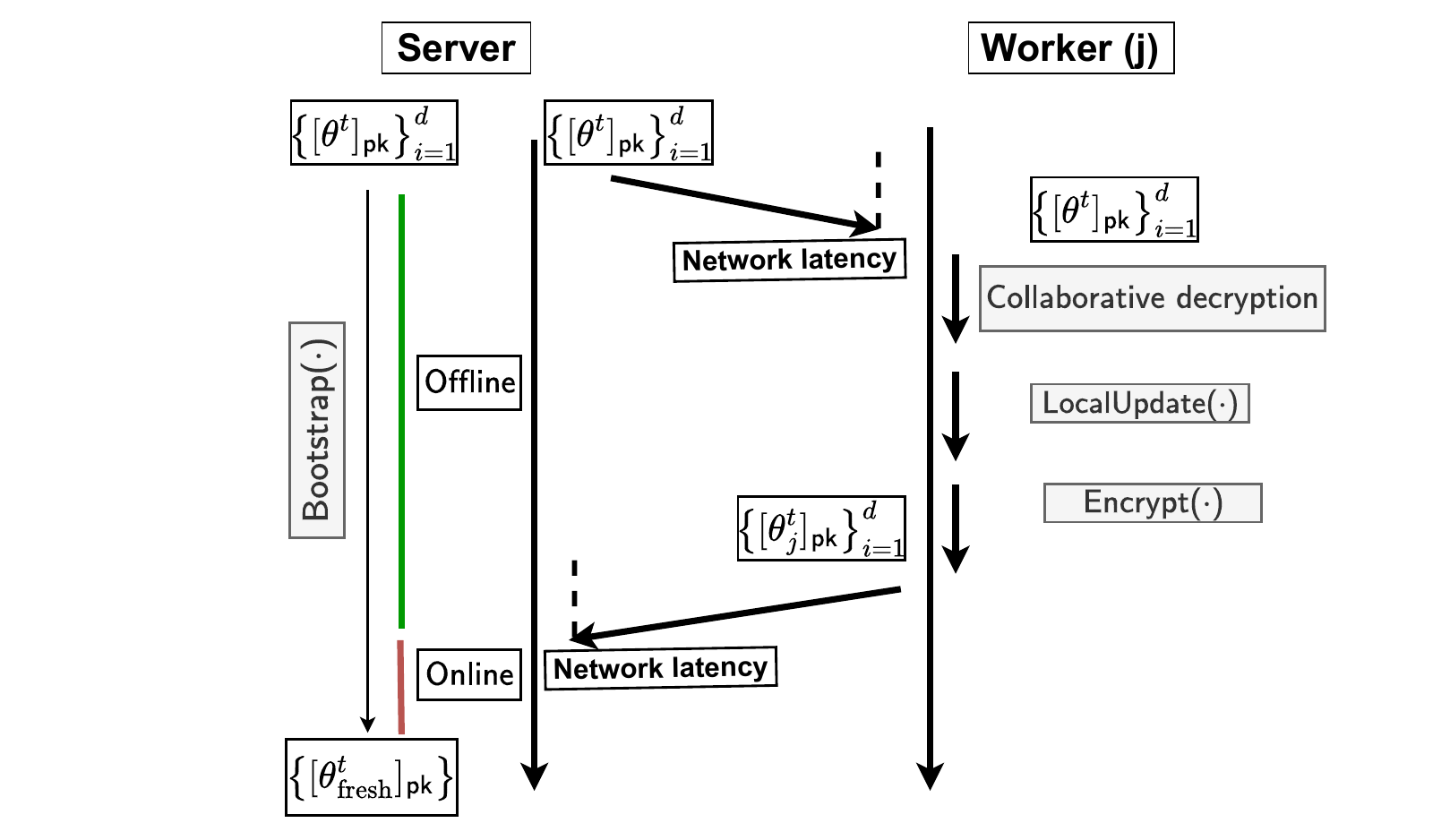}
    \caption{Server-worker communications, and actions illustrating the amortized cost of bootstrapping}
    \label{fig:temporal_diagram_bootstrapping}
\end{figure}

\paragraph{\textbf{Collaborative decryption}}
During the collaborative decryption phase, workers generate $t$-out-$t$ additive keys by (at least) $t$ active parties, using their respective $n$ shares of the initial $n$ additive secret keys (\cf{} equation \ref{eq:threshold_key_reconstruction}). 
Apart from revealing the active parties, the process of computing the $t$ partial decryptions does not require data exchanged other than what has already been exchanged during the setup. 
The last step is broadcasting the respective partial decryption by each of the $t$ active parties. 
A partially decrypted message consists of a degree $N$ polynomial with coefficients in $\mathbb{Z}_Q$. 
Thus, the size of data exchanged (for minimal $t$ active parties) is $t(t-1)N\log_2(Q)$ bits for a single ciphertext. 
Encrypting a model update requires $\left\lceil \frac{\mathsf{dim}(\theta)}{N/2} \right\rceil$ ciphertexts, which are 3 to 15 ciphertexts for our models (see table \ref{tab:nn_models}). Table \ref{tab:collab_decryption_overhead} illustrates the size of the exchanged data during the collaborative decryption and the runtime necessary to generate the shares for several values of $\log(N)$, and different numbers of active parties ($t$).

\begin{table}[htbp]
\scriptsize
  \caption{Bandwidth consumption (GB) and communication overhead (in seconds) for collaborative decryption of a single ciphertext encrypting $\frac{N}{2}$ model parameters, $\log_2(Q) = 60$.}
  \centering
  \begin{tabular}{|c|c|c|c|c|c|c|}
    \hline
    \multirow{2}{*}{$\log_2(N)$} & \multirow{2}{*}{Performances} & \multicolumn{5}{|c|}{Active parties $t$} \\ \cline{3-7}
     &  & 10 & 20 & 50 & 75 & 100 \\
    \hline
    \multirow{2}{*}{11}
      & {Bandwidth} & 1.38 & 5.83 & 36.44 & 82.00 & 145.59 \\ \cline{3-7}
      & {Overhead} & 0.14 & 0.58 & 3.64 & 8.20 & 14.56 \\
    \hline
    \multirow{2}{*}{12}
      & {Bandwidth } & 2.76 & 11.67 & 72.97 & 164.18 & 291.18 \\ \cline{3-7}
      & {Overhead} & 0.30 & 1.17 & 7.34 & 16.51 & 29.27 \\
    \hline
    \multirow{2}{*}{13}
      & {Bandwidth} & 5.53 & 23.35 & 145.94 & 328.35 & 582.35 \\ \cline{3-7}
      & {Overhead} & 0.64 & 2.58 & 16.00 & 35.93 & 63.50 \\
    \hline
    \multirow{2}{*}{14}
      & {Bandwidth} & 11.06 & 46.69 & 291.88 & 656.70 & 1164.70 \\ \cline{3-7}
      & {Overhead} & 1.27 & 5.17 & 32.01 & 71.84 & 126.99 \\
    \hline
    \multirow{2}{*}{15}
      & {Bandwidth} & 22.11 & 93.34 & 583.75 & 1313.40 & 2329.40 \\ \cline{3-7}
      & {Overhead} & 2.40 & 9.90 & 61.38 & 137.67 & 243.53 \\
    \hline
  \end{tabular}
  \label{tab:collab_decryption_overhead}
\end{table}

\paragraph{\textbf{Local model update}} Local model updates are performed via a few epochs ($E$) of Stochastic Gradient Descent (SGD) on the local dataset. The well-established cost of SGD is primarily determined by the model architecture and the size of the training data. Table \ref{tab:local_update_cost} presents the measured runtime of these local updates for the \femnist{} and \folktables{} tasks, along with their corresponding models, for a single epoch ($E=1$). As the cost scales linearly with the number of epochs, the runtime for other epoch values can be easily deduced. Furthermore, if DP-SGD \cite{Abadi_2016} is used to update the models rather than SGD for post-training privacy (\cf{}, Section \ref{sec:appendix_filter_perf_with_dp}) the update cost drastically increases due to the per-sample gradient computation required to bound the training's sensitivity \cite{Dwork2006}. 

\begin{table}[ht]
    \scriptsize
    \caption{Runtime of local model updates for \femnist{} and \folktables{} tasks in seconds with $E=1$ for SGD and DP-SGD.}
    \small
    \centering
    \begin{tabular}{c|c|c|c}
        \hline
        \textbf{Task} & \textbf{Dataset Size} & \textbf{SGD} & \textbf{DP-SGD} \\ \hline
        \femnist{} CNN &  28.000 & 2.71 & 5.87\\ \hline
        \folktables{} 3-layer MLP  & 40.000 & 3.02 & 7.32\\ \hline
    \end{tabular}
    \label{tab:local_update_cost}
\end{table}

\noindent Our analysis indicates that in cross-silo scenarios with a large number of workers (50, 75, or 100) and more than 3 local SGD epochs, worker-side computation becomes the dominant cost, surpassing the server's bootstrapping overhead. This implies that by the time workers transmit their encrypted updates, the server has either already completed bootstrapping or is nearing completion (\eg{}, 75\% done). Consequently, the cost of CKKS's bootstrapping to refresh the aggregated model's ciphertext $s$ (over which the SVM filter operates) has a practical impact on the total FL workflow under these realistic conditions. A concrete example is when $n=50$, $\log_2(N) = 14$ and $E=2$. The collaborative decryption requires around 32 seconds while the $\mathsf{LocalUpdate}$ takes around 5 seconds. Hence, the server receives the workers' updates after at least 37 seconds (without accounting for network latency). While the bootstrapping with $\log_2(N) = 14$ and $\log_2(Q) = 60$ takes around $30$ seconds \cite{ckks_boot_2022_bossuat}.

\section{Convergence Proofs}
\label{sec:appendix_conv_analysis}
In this section, we provide the proofs of the filtering error Lemmas, and our main convergence theorem. Recall that we denote the total number of workers in the system $n$, among whom $n_b$ are Byzantine, and $n_h$ are honest, and the Byzantine proportion is denoted $q= \frac{n_b}{n}$.
\begin{lemma}[Expected Filtering Error]
At round $t$, let $\theta^{t+1}_{\mathrm{ideal}}$ denote the ideal aggregated update, and $\theta^{t+1}_{\mathrm{real}}$ the real one. The filtering noise
\[
\delta_t := \theta^{t+1}_{\mathrm{real}} - \theta^{t+1}_{\mathrm{ideal}}
\]
satisfies
\[
\mathbb{E}[\|\delta_t\| \mid \theta^t] \le E_{\mathrm{filter}},
\]
where $E_{\text{filter}}$ is defined as:
\begin{align*}
E_{\text{filter}} = & n_h G \left( |1 - \mu_h| + \sigma_h \right) \\
& + n_b D_{\text{byz}} \left( \mu_b + \sigma_b \right)
\end{align*}
\end{lemma}
\begin{proof}
First, we recall that the ideal and real aggregated updates are:
\[
\theta^{t+1}_{\mathrm{ideal}} = \theta^t + \sum_{i \in H} \Delta_i^t
\]
\[
\theta^{t+1}_{\mathrm{real}} = \theta^t + \sum_{i \in H} P^t_i \Delta_i^t + \sum_{j \in B} P^t_j \Delta_j^t
\]
So,
\[
\delta_t = \sum_{i \in H} (P^t_i - 1)\Delta_i^t + \sum_{j \in B} P^t_j \Delta_j^t
\]
Then, we apply the triangle inequality,
\[
\|\delta_t\| \le
\left\| \sum_{i \in H} (P^t_i - 1)\Delta_i^t \right\|
+ \left\| \sum_{j \in B} P^t_j \Delta_j^t \right\|
\]
For honest clients, we obtain the following using Assumption 2; $\|\Delta_i^t\| \le G$:
\[
\left\| \sum_{i \in H} (P^t_i - 1)\Delta_i^t \right\|
\le \sum_{i \in H} |P^t_i - 1|\,\|\Delta_i^t\|
\le G \sum_{i \in H} |P^t_i - 1|
\]
Similarly, we obtain for Byzantine clients using Assumption 5; $\norm{\Delta^t_j} \leq D_{\text{byz}}$:
\[
\left\| \sum_{j \in B} P^t_j \Delta_j^t \right\|
\le \sum_{j \in B} |P^t_j|\,\|\Delta_j^t\|
\le D_{\text{byz}} \sum_{j \in B} |P^t_j|
\]
Then, we apply the Jensen's inequality to the conditional expectation as follows:
\begin{align*}
\mathbb{E}[\|\delta_t\| \mid \theta^t] \le & G \sum_{i \in H} \mathbb{E}[|P^t_i - 1|] + D_{\text{byz}} \sum_{j \in B} \mathbb{E}[|P^t_j|]
\end{align*}
Finally, using the bounds $\mathbb{E}[|P^t_i - 1|] \le \mathbb{E}[|P^t_i - \mu_h|] + |1-\mu_h| \le \sigma_h + |1-\mu_h|$ and $\mathbb{E}[|P^t_j|] \le \mathbb{E}[|P^t_j - \mu_b|] + |\mu_b| \le \sigma_b + \mu_b$ (assuming $\mu_b \ge 0$), we get:
\begin{align*}
\mathbb{E}[\|\delta_t\| \mid \theta^t] \le & G \sum_{i \in H} \left( |1 - \mu_h| + \sigma_h \right) + D_{\text{byz}} \sum_{j \in B} \left( \mu_b + \sigma_b \right) \\
= & n_h G \left( |1 - \mu_h| + \sigma_h \right) + n_b D_{\text{byz}} \left( \mu_b + \sigma_b \right)
\end{align*}
This corresponds to the stated bound.
\end{proof}

\begin{lemma}[Variance of the Filtering Error]
Under the same assumptions, the second moment of the filtering noise satisfies
\[
\mathbb{E}[\|\delta_t\|^2 \mid \theta^t] \le E_{\mathrm{filter}}^{(2)},
\]
where
\[
E_{\mathrm{filter}}^{(2)} =
2n_h^2\bigl[(1-\mu_h)^2 + \sigma_h^2\bigr]G^2
+ 2n_b^2(\mu_b^2 + \sigma_b^2)D_{\text{byz}}^2
\]
\end{lemma}

\begin{proof}
We Start from
\[
\delta_t = \sum_{i \in H} (P^t_i - 1)\Delta_i^t + \sum_{j \in B} P^t_j \Delta_j^t
\]
And use $\|a+b\|^2 \le 2(\|a\|^2 + \|b\|^2)$:
\begin{align*}
\mathbb{E}[\|\delta_t\|^2 \mid \theta^t]
&\le 2\,\mathbb{E}\!\left[\left\| \sum_{i \in H} (P^t_i - 1)\Delta_i^t \right\|^2\right] \\
&\quad + 2\,\mathbb{E}\!\left[\left\| \sum_{j \in B} P^t_j \Delta_j^t \right\|^2\right]
\end{align*}
We consider the first term. Using $(\sum_i x_i)^2 \le n_h \sum_i x_i^2$ and Assumption~2, we get:
\[
\begin{aligned}
\mathbb{E}\!\left[\left\| \sum_{i \in H} (P^t_i - 1)\Delta_i^t \right\|^2\right]
&\le n_h \sum_{i \in H} \mathbb{E}\big[\|(P^t_i - 1)\Delta_i^t\|^2\big] \\
&\le n_h G^2 \sum_{i \in H} \mathbb{E}[|P^t_i - 1|^2] \\
&= n_h^2 G^2\bigl[(1-\mu_h)^2 + \sigma_h^2\bigr]
\end{aligned}
\]
Similarly, we get for Byzantine workers using Assumption~3:
\[
\mathbb{E}\!\left[\left\| \sum_{j \in B} P^t_j \Delta_j^t \right\|^2\right]
\le n_b^2 D_b^2 (\mu_b^2 + \sigma_b^2)
\]
Finally, we combine both parts and obtain:
\[
\mathbb{E}[\|\delta_t\|^2 \mid \theta^t]
\le 2n_h^2\bigl[(1-\mu_h)^2 + \sigma_h^2\bigr]G^2
+ 2n_b^2(\mu_b^2 + \sigma_b^2)D_b^2
\]
\end{proof}

\begin{theorem}[Convergence with Filtering and Base Error]
Let the above assumptions hold, and let $\eta \le \frac{1}{2L}$. The iterations $\theta^t$ of the filtered FedAvg satisfy:
\begin{align*}
\E[F(\theta^{T})] - F(\theta^{*}) &\le (1-\eta\mu)^{T}(F(\theta^{\text{init}})-F(\theta^{*})) \\
& \quad + \frac{1}{\eta\mu}\Big( C_{g}(E_{\mathrm{base}} + E_{filter}) \\
& \qquad \qquad + 2L(E_{\mathrm{base}}^{(2)} + E_{filter}^{(2)}) \Big)
\end{align*}
\end{theorem}

\begin{proof}
We follow the standard descent lemma approach.

\textbf{Step 1: Apply Descent Lemma.}
Since $F$ is $L$-smooth, the Descent Lemma gives:
\begin{equation} \label{eq:smoothness}
F(\theta^{t+1}_{\mathrm{real}}) \le F(\theta^{t}) + \langle\nabla F(\theta^{t}), \theta^{t+1}_{\mathrm{real}}-\theta^{t}\rangle + \frac{L}{2}\norm{\theta^{t+1}_{\mathrm{real}}-\theta^{t}}^2
\end{equation}
Note that this inequality holds regardless of the nature of $\theta^t$ (real or ideal).

\textbf{Step 2: Decompose the update.}
The real update step $\theta^{t+1}_{\mathrm{real}} - \theta^{t}$ can be decomposed into the ideal gradient step and the two error terms (Base Error $\mathbf{B}_t$ and Filtering Error $\delta_t$).
By definition, $\theta^{t+1}_{\mathrm{real}} - \theta^t = (\theta^{t+1}_{\mathrm{ideal}} - \theta^t) + \delta_t$.
From the definition of $\mathbf{B}_t$ (Assumption 3), we have $(\theta^{t+1}_{\mathrm{ideal}} - \theta^t) = \mathbf{B}_t - \eta\nabla F(\theta^t)$.
Substituting this in gives the full decomposition:
\begin{equation} \label{eq:progress}
\theta^{t+1}_{\mathrm{real}} - \theta^{t} = -\eta\nabla F(\theta^{t}) + \mathbf{B}_t + \delta_t
\end{equation}

\textbf{Step 3: Bound the inner product term.}
Substitute \eqref{eq:progress} into the inner product term of \eqref{eq:smoothness} and take the conditional expectation given $\theta^t$:
\begin{align*}
\E[\langle\nabla F(\theta^{t}), &\theta^{t+1}_{\mathrm{real}}-\theta^{t}\rangle | \theta^t] \\
&= \E[\langle\nabla F(\theta^{t}), -\eta\nabla F(\theta^{t}) + \mathbf{B}_t + \delta_t \rangle | \theta^t] \\
&= -\eta\norm{\nabla F(\theta^{t})}^2 + \langle\nabla F(\theta^{t}), \E[\mathbf{B}_t | \theta^t]\rangle \\
&\quad + \langle\nabla F(\theta^{t}), \E[\delta_t | \theta^t]\rangle
\end{align*}
Apply Cauchy-Schwarz ($\langle a, b \rangle \le \norm{a}\norm{b}$) to both error terms:
\begin{align*}
&\le -\eta\norm{\nabla F(\theta^{t})}^2 + \norm{\nabla F(\theta^{t})} \norm{\E[\mathbf{B}_t | \theta^t]} \\
&\quad + \norm{\nabla F(\theta^{t})} \norm{\E[\delta_t | \theta^t]}
\end{align*}
Apply Jensen's inequality ($\norm{\E[X]} \le \E[\norm{X}]$) to the filtering term:
\begin{align*}
&\le -\eta\norm{\nabla F(\theta^{t})}^2 + \norm{\nabla F(\theta^{t})} \norm{\E[\mathbf{B}_t | \theta^t]} \\
&\quad + \norm{\nabla F(\theta^{t})} \E[\norm{\delta_t} | \theta^t]
\end{align*}
Now apply Assumptions 2, 4, and 5:
\begin{equation} \label{eq:inner_prod_bound}
\le -\eta\norm{\nabla F(\theta^{t})}^2 + C_g E_{\mathrm{base}} + C_g E_{filter}
\end{equation}

\textbf{Step 4: Bound the quadratic term.}
We bound the expected squared norm of the progress step \eqref{eq:progress}. We use the inequality $\norm{a+b}^2 \le 2\norm{a}^2 + 2\norm{b}^2$ repeatedly:
\begin{align}
\E[\norm{\theta^{t+1}_{\mathrm{real}}}&{-\theta^{t}}^2 | \theta^t] \notag \\
&= \E[\norm{(-\eta\nabla F(\theta^{t})) + (\mathbf{B}_t + \delta_t)}^2 | \theta^t] \notag \\
&\le 2\E[\norm{-\eta\nabla F(\theta^{t})}^2] + 2\E[\norm{\mathbf{B}_t + \delta_t}^2 | \theta^t] \notag \\
&\le 2\eta^2\norm{\nabla F(\theta^{t})}^2 + 2(2\E[\norm{\mathbf{B}_t}^2] + 2\E[\norm{\delta_t}^2]) \notag \\
&\le 2\eta^2\norm{\nabla F(\theta^{t})}^2 + 4E_{\mathrm{base}}^{(2)} + 4E_{\mathrm{filter}}^{(2)} \label{eq:quadratic_bound}
\end{align}

\textbf{Step 5: Combine bounds.}
Substitute \eqref{eq:inner_prod_bound} and \eqref{eq:quadratic_bound} into the conditional expectation of \eqref{eq:smoothness}:
\begin{align*}
\E[F(\theta^{t+1}_{\mathrm{real}}) &| \theta^t] \le F(\theta^{t}) - \eta\norm{\nabla F(\theta^{t})}^2 \\
&\quad + C_g E_{\mathrm{base}} + C_g E_{filter} \\
&\quad + \frac{L}{2}\left(2\eta^2\norm{\nabla F(\theta^{t})}^2 + 4E_{\mathrm{base}}^{(2)} + 4E_{\mathrm{filter}}^{(2)}\right) \\
&= F(\theta^{t}) - \eta(1 - L\eta)\norm{\nabla F(\theta^{t})}^2 \\
&\quad + C_g(E_{\mathrm{base}} + E_{filter}) \\
&\quad + 2L(E_{\mathrm{base}}^{(2)} + E_{\mathrm{filter}}^{(2)})
\end{align*}
Subtract $F(\theta^*)$ from both sides:
\begin{align} \label{eq:combined_ineq}
\E[F(\theta^{t+1}_{\mathrm{real}}) - F^* | \theta^t] &\le (F(\theta^{t}) - F^*) \\
&\quad - \eta(1 - L\eta)\norm{\nabla F(\theta^{t})}^2 \notag \\
&\quad + C_g(E_{\mathrm{base}} + E_{filter}) \notag \\
&\quad + 2L(E_{\mathrm{base}}^{(2)} + E_{\mathrm{filter}}^{(2)}) \notag
\end{align}

\textbf{Step 6: Use strong convexity and assumptions.}
Since $\eta \le 1/(2L)$, we have $(1-L\eta) \ge 1/2$.
By $\mu$-strong convexity (Assumption 1), $\norm{\nabla F(\theta^{t})}^2 \ge 2\mu(F(\theta^{t})-F^*)$.
Combining these:
\[
-\eta(1-L\eta)\norm{\nabla F(\theta^{t})}^2 \le -\frac{\eta}{2}\norm{\nabla F(\theta^{t})}^2
\]
\[
\le -\frac{\eta}{2}(2\mu(F(\theta^{t})-F^*)) = -\eta\mu(F(\theta^{t})-F^*)
\]
Substitute this into \eqref{eq:combined_ineq}:
\begin{align*}
\E[F(\theta^{t+1}_{\mathrm{real}}) &- F^* | \theta^t] \le (1-\eta\mu)(F(\theta^{t}) - F^*) \\
&\quad + C_g(E_{\mathrm{base}} + E_{filter}) \\
&\quad + 2L(E_{\mathrm{base}}^{(2)} + E_{\mathrm{filter}}^{(2)})
\end{align*}

\textbf{Step 7: Unroll the recursion.}
Let $R$ be the residual error constant:
\[
R = C_g(E_{\mathrm{base}} + E_{filter}) + 2L(E_{\mathrm{base}}^{(2)} + E_{\mathrm{filter}}^{(2)})
\]
Take the total expectation (which is valid by the tower property and because $R$ is a constant upper bound):
\[
\E[F(\theta^{t+1}_{\mathrm{real}}) - F^*] \le (1-\eta\mu)\E[F(\theta^{t}) - F^*] + R
\]
Let $\mathcal{E}_t = \E[F(\theta^{t}) - F^*]$. The recursion is $\mathcal{E}_{t+1} \le (1-\eta\mu)\mathcal{E}_t + R$.
Unrolling this recursion from $t=0$ (init) to $t=T-1$:
\begin{align*}
\mathcal{E}_T &\le (1-\eta\mu)^T \mathcal{E}_0 + R \sum_{k=0}^{T-1} (1-\eta\mu)^k \\
&\le (1-\eta\mu)^T \mathcal{E}_0 + R \sum_{k=0}^{\infty} (1-\eta\mu)^k \\
&= (1-\eta\mu)^T \mathcal{E}_0 +  \frac{R}{1 - (1-\eta\mu)} \\
&= (1-\eta\mu)^T \left(F(\theta^{\text{init}})-F(\theta^*)\right) + \frac{R}{\eta\mu}
\end{align*}
Replacing $R$ with its definition gives the final result.
\end{proof}

\section{Single Filter Multiple Attacks}
\label{sec:appendix_single_filter_multiple_attacks}
We explore the feasibility of integrating multiple filtering capabilities into a unified filtering model, addressing the research question: \emph{Can a single SVM filtering model accurately defend against more than one Byzantine behavior?}. Our empirical results show that this is indeed achievable: a single SVM filter can accurately filter multiple Byzantine activities, preserving convergence without requiring attack-specific models.

\subsection{\textbf{Methodology}}
We aim to explore the defense against a combination of Byzantine activities, defined by properties $\{P_1, \dots, P_r\}$, using a single filtering model. To achieve this, we construct a combined dataset, $\mathcal{D}_{P_{\vee}}$, by merging the individual shadow datasets, $\{\mathcal{D}_{P_i}\}^r_{i=1}$, each representing a specific Byzantine activity $P_i$. Honest update differences, $\{\Delta_i\}$, are generated using genuine local SGD on unaltered data. Byzantine updates, $\{\bar{\Delta}\}$, are generated by applying the corresponding Byzantine behavior function, $\mathsf{DisHonestUpdate}_{P_j}(\cdot)$, for each $j \in \{1, \cdots, r\}$, and are labeled negatively (0). The resulting dataset,
$$
\mathcal{D}_{P_{\vee}} = \bigcup_{i=1}^{r} \mathcal{D}_{P_i}
$$
is used to train a classifier. This classifier identifies updates that exhibit \textbf{at least one} of the Byzantine behaviors in $\{P_1, \cdots, P_r\}$, corresponding to the logical OR. That is, $P_{\vee} = (P_1 \vee \cdots \vee P_r)$.

\subsection{Experimental Setup}
We construct a combined dataset by merging shadow updates from the gradient-ascent and label-shuffling attacks, and train a single SVM on this joint dataset. We evaluate its predictive performance in two ways:
(i) its ability to detect whether either Byzantine behavior occurs ($P_1 \cup P_2$), and
(ii) its ability to detect each behavior individually (predicting $P_1$ alone and $P_2$ alone).
We then compare these results against SVMs trained exclusively on a single attack type. Table \ref{tab:dual_property_performances} reports all performance metrics.

Subsequently, we conduct \fedavg{}-based federated learning experiments with our dual property filter on the server side in the following environment and metrics:
\begin{itemize}[leftmargin=*]
    \item{\textbf{Environment.}} This environment aims to simultaneously assess the detection capabilities of two types of attackers. For this purpose, we set 6/10 Byzantine workers, among whom 3 perform gradient-ascent attacks and the remaining 3 perform label-shuffling attacks.
    This setup induces three distinct worker groups, each with a separate objective. The goal of our filtering method is to ensure that the honest worker group successfully achieves its objective. That is genuine convergence.
    \item{\textbf{Metrics.}} We track the test accuracy on genuine data to assess the defense against both attacker groups,
\end{itemize}

\paragraph{\textbf{Findings}} 
Table \ref{tab:dual_property_performances} shows that a single SVM filter, trained on both label-shuffling and gradient-ascent samples, effectively detects both Byzantine behaviors with high $\mathsf{F1}$-scores (93.7\% for label-shuffling, 92.4\% for gradient-ascent), comparable to specialized SVMs (90.3\% and 94.1\%, respectively). Although the combined model uses 59 support vectors compared to 48 and 27 for the individual models, this is fewer than deploying two separate models (75 SVs total) and suggests the complexity of the separating hyperplane required for dual detection is not substantially increased.

\begin{table}[ht!]
\centering
\caption{Single-attack vs.\ dual-attack SVM filtering performance.}
\label{tab:dual_property_performances}
\small
\begin{tabular}{lc}
\toprule
\textbf{Model} & \textbf{Value} \\
\midrule
\multicolumn{2}{l}{\textbf{SVM 1: Label Shuffling}} \\
\quad Kernel degree & 1 \\
\quad Support vectors & 55 \\
\quad F1-score & 89.2\% \\
\quad Layer  & C2 \\

\midrule
\multicolumn{2}{l}{\textbf{SVM 2: Gradient Ascent}} \\
\quad Kernel degree & 1 \\
\quad Support vectors & 27 \\
\quad F1-score & 94.1\% \\
\quad Layer & Last FC \\

\midrule
\multicolumn{2}{l}{\textbf{Dual-attack SVM}} \\
\quad Kernel degree & 1 \\
\quad Support vectors & 83 \\
\quad F1 (Both) & 91.5\% \\
\quad F1 Gradient Ascent & 88.4\% \\
\quad F1 Label-shuffling & 86.1\% \\
\quad Layer & \textsf{Input} \\
\bottomrule
\end{tabular}
\end{table}

\begin{figure}
    \centering
    \includegraphics[width=0.9\linewidth]{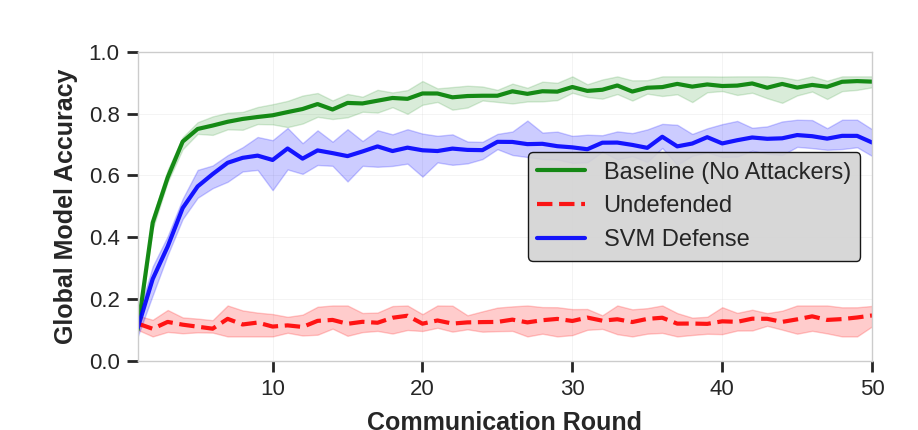}
    \caption{Convergence of our Byzantine filtering approach under two Byzantine groups with a single SVM filter.}
    \label{fig:combined_attack}
\end{figure}

\begin{figure}
    \centering
    \includegraphics[width=0.9\linewidth]{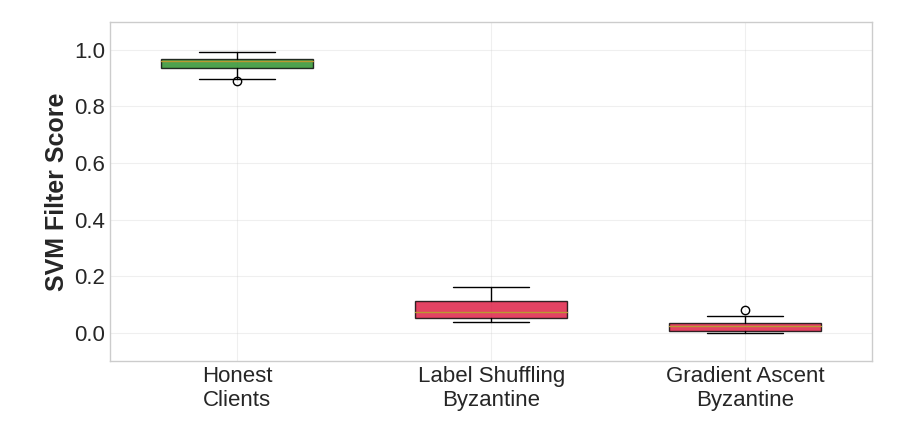}
    \caption{Filter values distribution on honest, label shuffling attackers, and gradient-ascent attackers.}
    \label{fig:combined_attack_filters}
\end{figure}

As shown in Figure \ref{fig:combined_attack}, the combined presence of gradient-ascent and label-shuffling attackers severely disrupts the convergence dynamics. Their consistent negation of the honest workers’ learning signal overwhelms the aggregation process, preventing the global model from approaching the desired optimum. When SVM-based filtering is applied, both Byzantine groups are effectively down-weighted, as illustrated in Figure \ref{fig:combined_attack_filters}. Thus, the global model’s test accuracy converges back toward the baseline, without adversarial participants.

\section{Filtering Differentially-Private Updates}
\label{sec:appendix_filter_perf_with_dp}

While secure aggregation techniques safeguard privacy during the training phase, they do not mitigate post-deployment black-box privacy attacks ~\cite{lyu2022privacy, shokri2017membership, rigaki2020survey}. To ensure strong privacy guarantees against inference attacks occurring after model deployment, differential privacy is the standard approach. As such, differentially-private training is often co-implemented with secure aggregation in federated learning to maintain privacy throughout the model’s entire lifecycle.  

Differential privacy \cite{10.1007/11787006_1} is a mathematical framework that defines, quantifies, and provides privacy at the level of individual data samples. Informally, it ensures that the output of a randomized analysis over a dataset is statistically similar regardless of whether or not any particular individual's data is included in the dataset. Formally, a randomized mechanism $\mathcal{M}$ satisfies $(\epsilon, \delta)$-differential privacy if, for any two neighboring datasets $D$ and $D'$ differing in exactly one sample, and for any of output distribution $S$:

$$
    \Pr[\mathcal{M}(D) \in S] \leq e^{\epsilon} \Pr[\mathcal{M}(D') \in S] + \delta
$$

Here, $\epsilon > 0$ controls the privacy loss, with smaller values ensuring stronger privacy, while $\delta$ represents the probability of violating the $\epsilon$ privacy guarantee.

The standard approach to achieve DP in SGD-based machine learning is via the DP-SGD algorithm by Abbadi \etal \cite{Abadi_2016}. DP-SGD operates by computing per-sample gradients, clipping them to limit their influence, averaging the clipped gradients, and finally adding Gaussian noise calibrated to the desired privacy level before updating the model with the resulting gradient vector. The clipping step constrains the \emph{sensitivity} of the training process, thereby limiting the effect of any single data point on the model’s optimization trajectory. Meanwhile, the addition of Gaussian noise ensures that the overall procedure satisfies $(\epsilon, \delta)$-DP by leveraging the Gaussian mechanism \cite{Dwork2006} once sensitivity is bounded.

While differential privacy offers strong privacy guarantees against attacks targeting individual samples, such as membership or attribute inference, it does not defend against property inference attacks \cite{rigaki2020survey, melis2018exploiting}. This is because PIAs infer global patterns (properties) in the dataset rather than specific individual-related information, hence falling outside the scope of DP. Indeed, the idea of \emph{property} is not captured by the notions of \emph{neighboring datasets} nor \emph{sensitivity}. Therefore, since our filtering method follows the white-box PIA blueprint \cite{parisot2021propertyinferenceattacksconvolutional}, it is expected to remain effective when applied to differentially-private and encrypted model updates.

\subsection{Experimental Setup}
We assess the filtering capabilities of our pre-trained SVM filters—trained on non-DP shadow updates—when applied to shadow updates produced via DP-SGD. Specifically, we generate 300 shadow updates following the procedure in Algorithm \ref{alg:filtering_data_generation}, replacing the standard $\mathsf{LocalUpdate}(\cdot)$, for both honest and Byzantine, with a differentially private variant that uses DP-SGD. Among these 300 updates, 100 are produced with weak privacy guarantees ($\epsilon = 5.0$), another 100 with moderate privacy guarantees ($\epsilon = 1.0$), and the remaining 100 with strong privacy guarantees ($\epsilon = 0.1$). SPCA+SVM inference is performed on each group of these shadow-updates, and the results are reported in Figure \ref{fig:filters_under_dp}.

\subsection{Findings}
While the practical DP levels commonly used in FL—combined with secure aggregation—such as $\epsilon = 1.0$ incur little to no degradation in filtering performance compared to the weak-privacy setting ($\epsilon = 5.0$) or the non-private baselines discussed in Section \ref{sec:experimental}, the strong-privacy regime introduces substantial distortions for both honest and Byzantine shadow updates. This behavior is expected. Indeed, enforcing strong privacy requires aggressive gradient clipping, which significantly alters the geometry of the updates. As a consequence, the discriminative signal captured by the SPCA components is weakened, reducing the separability between honest and adversarial directions. In other words, when gradients are heavily clipped and noised, the features that our SVM filters rely on are suppressed, ultimately limiting the effectiveness of the filtering mechanism under stringent privacy constraints.

\begin{figure}[ht!]
    \centering
    \captionbox*{\small{(a) Backdoor: \GTSRB{} VGG11}}{\includegraphics[width=0.35\textwidth]{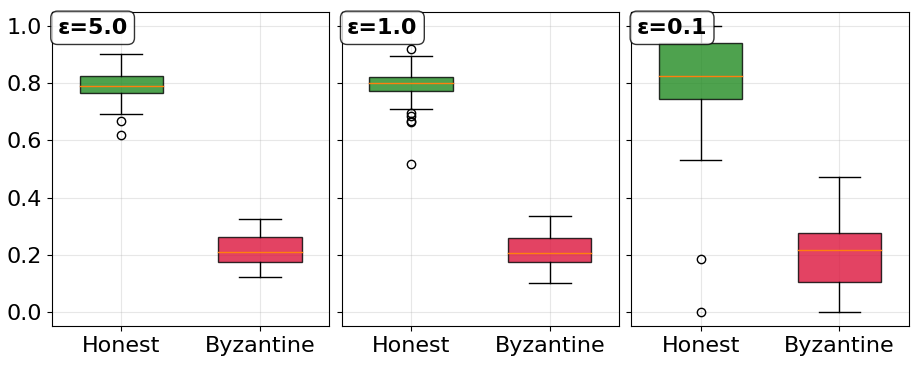}}
    \hfill 
    \captionbox*{\small{(b) Grad-ascent: \cifar{}{} ResNet-14}}{\includegraphics[width=0.35\textwidth]{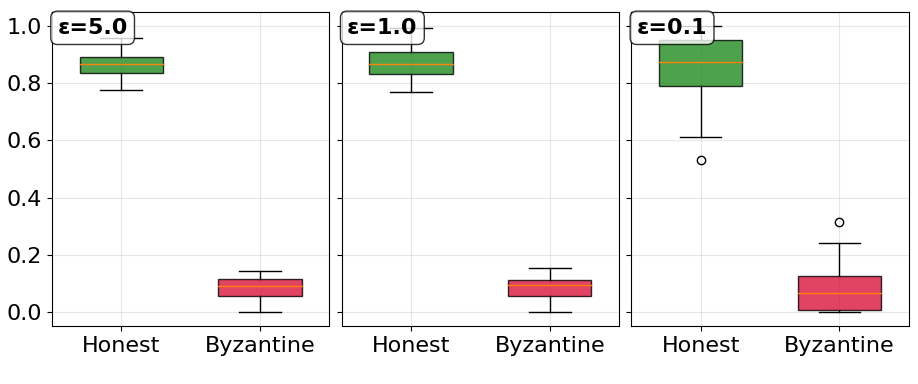}}
    \hfill 
    \captionbox*{\small{(c) Label-flipping: \acsincome{} 3-MLP}}{\includegraphics[width=0.35\textwidth]{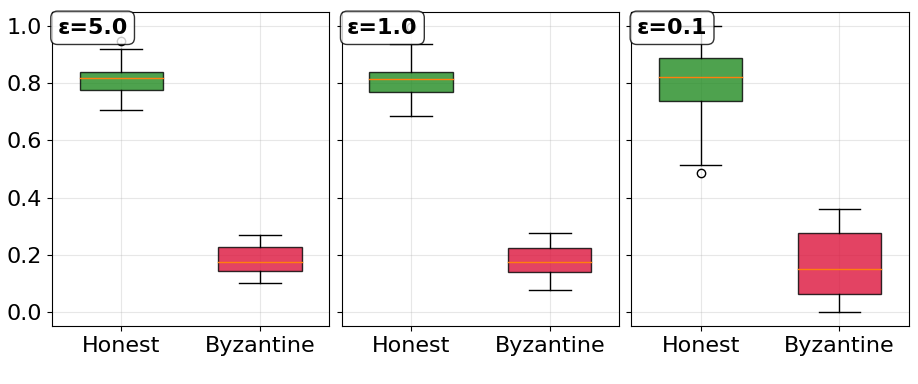}}
    \captionbox*{\small{(c) Label-shuffling: \femnist{} LeNet-5}}{\includegraphics[width=0.35\textwidth]{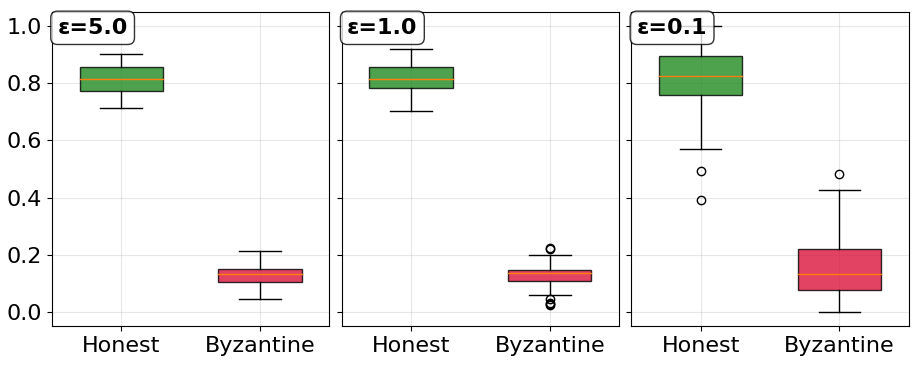}}
    
    \caption{SVM filters performance on differentially-private shadow updates with three levels of DP.}
    \label{fig:filters_under_dp}
\end{figure}

\end{document}